\documentclass[11pt]{article}

\usepackage{fullpage}

\usepackage{amsmath,amsfonts,bm}

\def\eqref#1{equation~\ref{#1}}

\def\1{\bm{1}}

\def\vn{{\bm{n}}}

\DeclareMathAlphabet{\mathsfit}{\encodingdefault}{\sfdefault}{m}{sl}
\SetMathAlphabet{\mathsfit}{bold}{\encodingdefault}{\sfdefault}{bx}{n}

\def\gC{{\mathcal{C}}}
\def\gD{{\mathcal{D}}}

\def\gR{{\mathcal{R}}}

\DeclareMathOperator*{\argmax}{arg\,max}
\DeclareMathOperator*{\argmin}{arg\,min}

\usepackage[utf8]{inputenc} 
\usepackage[T1]{fontenc}    
\usepackage{hyperref}       
\usepackage{url}            
\usepackage{booktabs}       
\usepackage{amsfonts}       
\usepackage{nicefrac}       
\usepackage{microtype}      
\usepackage{xcolor}         

\usepackage{algorithm}
\usepackage{algorithmic}

\usepackage{algorithm} 
\usepackage{algorithmic}
\usepackage{microtype}
\usepackage{graphicx}
\usepackage{subfigure}

\usepackage{natbib} 
\usepackage{mathtools} 
\usepackage{booktabs} 
\usepackage{tikz} 

\usepackage[capitalize,noabbrev]{cleveref}

\usepackage{amsmath}
\usepackage{amssymb}
\usepackage{mathtools}
\usepackage{amsthm}
\usepackage{dsfont}

\theoremstyle{plain}
\newtheorem{assumption}{Assumption}

\newtheorem{definition}{Definition}
\newtheorem{remark}{Remark}
\newtheorem{theorem}{Theorem}

\newtheorem{lemma}{Lemma}

\def \bC {\displaystyle\sC}
\def \bD {\displaystyle\sD}
\def \bR {\displaystyle\sR}
\def \bT {\mathcal{T}}
\def \bI {\bold{I}}

\def \bD {\mathcal{D}}
\def \bR {\mathcal{R}}
\def \bE {\mathbb{E}}
\def \bC {\mathcal{C}}
\def \bbC {\mathcal{C}}

\def \bP {\mathcal{P}}
\def \bn {\displaystyle\vn}
\def \bc {\mathcal{C}}
\def \ba {\bold{a}}
\def \bone {\mathds{1}}
\def \bE {\mathds{E}}
\def \bB {\mathbb{B}}
\def \bN {\mathcal{N}}
\def \bbN {N}

\usepackage[textsize=tiny]{todonotes}
\usepackage[american]{babel}

\usepackage{natbib} 
\usepackage{mathtools} 
\usepackage{tikz} 

\usepackage {xr}
\makeatletter
\newcommand*{\addFileDependency}[1]{
  \typeout{(#1)}
  \@addtofilelist{#1}
  \IfFileExists{#1}{}{\typeout{No file #1.}}
}
\makeatother

\title{\textbf{Adversarial Attacks on Online Learning to Rank with Stochastic Click Models}}

\author{ Zichen Wang \footnotemark[1] \\
\and 
Rishab Balasubramanian \footnotemark[2]\\
\and 
Hui Yuan \footnotemark[3]\\
\and 
Chenyu Song \footnotemark[2]\\
\and 
Mengdi Wang \footnotemark[3]\\
\and 
Huazheng Wang \footnotemark[2]
}
\date{}
\begin{document}
\maketitle

\renewcommand{\thefootnote}{\fnsymbol{footnote}}
\footnotetext[1]{Southwest University, \texttt{swuzcwang@163.com}}
\footnotetext[2]{Oregon State University, \texttt{\{balasuri,songchen,huazheng.wang\}@oregonstate.edu}}
\footnotetext[3]{Princeton University, \texttt{\{huiyuan,mengdiw\}@princeton.edu}}

\begin{abstract}
  We propose the first study of adversarial attacks on online learning to rank. The goal of the adversary is to misguide the online learning to rank algorithm to place the target item on top of the ranking list linear times to time horizon $T$ with a sublinear attack cost. We propose generalized list poisoning attacks that perturb the ranking list presented to the user. This strategy can efficiently attack any no-regret ranker in general stochastic click models. Furthermore, we propose a click poisoning-based strategy named attack-then-quit that can efficiently attack two representative OLTR algorithms for stochastic click models. We theoretically analyze the success and cost upper bound of the two proposed methods. Experimental results based on synthetic and real-world data further validate the effectiveness and cost-efficiency of the proposed attack strategies.
\end{abstract}

\section{Introduction}


Online learning to rank (OLTR) \citep{grotov2016online} formulates learning to rank \citep{liu2009learning}, the core problem in information retrieval, as a sequential decision-making problem. OLTR is a family of online learning solutions that exploit implicit feedback from users (e.g., clicks) to directly optimize parameterized rankers on the fly. It has drawn increasing attention in recent years \citep{Kveton2015CascadingBL,zoghi2017onlineLT,Lattimore2018TopRankAP,Oosterhuis_2018,wang2019variance,jia2021pairrank} due to its advantages over traditional offline learning-based solutions and numerous applications in web search and recommender systems~\citep{liu2009learning}. 

To effectively utilize users' click feedback to improve the
quality of ranked lists, one line of OLTR studied bandit-based algorithms under different click models. 
In each iteration, the algorithm presents a ranked list of $K$ items selected from $L$ candidates based on its estimation of the user's interests. The ranker observes the user's click feedback and updates these estimates accordingly. 
Different users may examine and click on the ranking list differently, and how the user interacts with the item list is called the \textit{click model}. Many works have been dedicated to establishing OLTR algorithms in the cascade model \citep{Kveton2015CascadingBL,Kveton2015CombinatorialCB,Zong2016CascadingBF,li2016contextual,Vial2022MinimaxRF}, the position-based model \citep{Lagre2016MultiplePlayBI} and the dependent click model \citep{Katariya2016DCMBL,Liu2018ContextualDC}. However, these algorithms are ineffective when employed under a different click model. To overcome this bottleneck, \citet{zoghi2017onlineLT,Lattimore2018TopRankAP,Li2019OnlineLT} proposed OLTR algorithms with general stochastic click models that cover the aforementioned click models. 

There has been a huge interest in developing robust and trustworthy information retrieval systems \citep{Golrezaei2021LearningPR,Ouni2022DeepLF,Sun2016TheCO}, and understanding the vulnerability of OLTR algorithms to adversarial attacks is an essential step towards the goal. Recently, several works explored adversarial attacks on multi-armed bandits \citep{jun2018adversarial,liu2019data} and linear bandits \citep{garcelon2020adversarial,wang2022linear} where the system recommends one item to the user in each round. The idea of the poisoning attack is to lower the rewards of the non-target item to misguide the bandit algorithm to recommend the target item using cost sublinear to time horizon $T$. In online ranking, we consider the goal of the adversary as misguiding the algorithm to rank the target item on top of the ranking list linear times ($T-o(T)$) with sublinear attack cost ($o(T)$).
However, it is hard to directly extend the attack strategy on multi-armed bandits to OLTR since the click model is a black box to the adversary.

In this paper, we propose the first study of adversarial attacks on OLTR with stochastic click models. We study two threat models: click poisoning attacks where the adversary manipulates the rewards the user sends back to the ranking algorithm, and list poisoning attacks where the adversary perturbs the ranking list presented to
the user. We first propose a generalized list poisoning attack strategy that can \textit{efficiently attack any no-regret ranker} for stochastic click models. The adversary perturbs the ranking list presented to the user and pretends the click feedback represents the user's interests in the original ranking list. This guarantees the feedback always follows the unknown click model, making the attack \textit{stealthy}. Furthermore, we propose a click poisoning-based strategy named attack-then-quit that can efficiently attack two representative OLTR algorithms for stochastic click models, i.e., BatchRank \citep{zoghi2017onlineLT} and TopRank \citep{Lattimore2018TopRankAP}. Our theoretical analysis guarantees that the proposed methods succeed with sublinear attack cost. We empirically evaluate the proposed methods against several  OLTR algorithms on synthetic data and a real-world dataset under different click models. Our experimental results validated  the theoretical analysis of the effectiveness and cost-efficiency of the two proposed attack algorithms.

\section{Preliminaries}
\subsection{Online learning to rank}
We denote the total item set with $L$ items as $\gD = \{a_1,...,a_L\}$. Let $\Pi_{K}(\gD) \subset \gD^K$ stands for all $K$-tuples with different elements from $\gD$. At each round $t$, the ranker would present a length-$K$ ordered list $\gR_t = (\ba_{1}^t,...,\ba_{K}^t) \in \Pi_{K}(\gD) $ to the user, where $\ba_{k}^t$ is the item placed at the $k$-th position of $\gR_t$. Generally, $K$ is a constant much smaller than $L$. When the user observes the provided list, he/she returns click feedback $\gC_t = (\gC_{1}^t,..., \gC_{L}^t)$ to the ranker where $\bC^t_k = 1$ stands for user click on item $a_k$. Note that $a_k\not\in \bR_t$ can not be observed by the user, thus its click feedback in round $t$ is $\bc_k^t = 0$. 
The attractiveness score represents the probability the user is interested in item $a_k$, and is defined as  $\alpha(a_k) \in [0,1]$, which is unknown to the ranker. Without loss of generality, we suppose $\alpha(a_1)>,...,>\alpha(a_L)$ where $a_1$ is the most attractive item and $a_L$ is the least attractive item. 

\subsection{Stochastic click models}

In this paper, we consider the general stochastic click models studied by \citet{zoghi2017onlineLT,Lattimore2018TopRankAP}, where the conditional probability that the user clicks on position $k$ in round $t$ is only related to $\bR_t$. This implies there exists an unknown function that satisfies
\begin{align}
    P(\bC_s^t=1\ \vert\ \bR_t = \bR,\ \ba_k^t = a_s) = v(\bR,\ba_k^t,k).
\end{align}
The key problem of OLTR is to present the optimal list $\bR^* = (a_1,...,a_K)$ to the user for per-round click number maximization. 
The optimal list is unique due to the attractiveness of items is unique.

\begin{assumption} [Assumption 2 of \citep{Lattimore2018TopRankAP}]\label{supassumption} Due to the user does not observe items in position $\not \in \bR_t$, we assume the ranker can achieve maximum expected number of clicks in round $t$ if and only if $\bR_t = \bR^*$, i.e. 
\begin{align}
\max_{\bR\in\Pi_K(\bD)} \sum_{k=1}^K v(\bR,\ba_k^t,k) = \sum_{k=1}^K v(\bR^*,\ba_k^t,k).
\end{align}
\end{assumption}

\begin{definition} [Cumulative regret]\label{supdefinition}
 The performance of a ranker can be evaluated by the cumulative regret, defined as
\begin{align}\label{5}
\begin{split}
\nonumber
    R(T) = \bE \bigg[ T \sum_{k = 1}^K v(\gR^*,\ba_k^t,k) - \sum_{t = 1}^T \sum_{k = 1}^K v(\gR_t,\ba_{k}^t,k) \bigg].
\end{split}
\end{align}
Note that if Assumption \ref{supassumption} holds, $\bR^*$ can uniquely maximize $\sum_{k = 1}^K v(\gR_t,\ba_{k}^t,k)$, and every $\bR_t \not = \bR^*$ leads to non-zero regret. 
\end{definition}

We present two classic click models \citep{Chuklin2015ClickMF,Richardson2007PredictingCE,Craswell2008AnEC} that are special instances of the stochastic click models.

\paragraph{Position-based model.} The position-based model \citep{Richardson2007PredictingCE} assumes the examination probability of the $k$-th position in list $\gR_t$ is a constant $\chi(k) \in [0,1]$. In each round, the user receives the ordered list $\gR_t$. He/she would examine position $ k $ with probability $\chi(k)$. If position $k$ is examined then the user would click item $\ba_{k}^t$ with probability $\alpha(\ba_k^t)$. Hence, the probability of item $\ba_k^t$ is clicked by the user is
    \begin{align}
v(\gR_t,\ba_{k}^t,k)  = \chi(k)\alpha(\ba_{k}^t).
    \end{align}
Note that the examination probability of items not in $\gR_t$ is $0$. Hence, the expected number of clicks in round $t$ is
    \begin{align}\label{PBM}
         \sum^{K}_{k = 1} v(\gR_t,\ba_{k}^t,k) = \sum^{K}_{k = 1} \chi(k) \alpha(\ba_{k}^t).
    \end{align}
The examination probabilities of the first $K$ positions  are assumed to follow $\chi(1) > ... > \chi(K)$ \citep{Chuklin2015ClickMF}. The maximum number of clicks in each round is $K$.
    
\paragraph{Cascade model.}
    In the cascade model \citep{Craswell2008AnEC}, the user examines the items in $\gR_t$ sequentially from $\ba_{1}^t$. The user continues examining items until they find an item $\ba_{k}^t$ attractive or they reach the end of the list. If the user finds $\ba_{k}^t$ attractive, they would click on it and stop examining further.

    According to the above description, the examination probability of position $k$ equals the probability of none of the items in the first $k-1$ positions in $\gR_t$ can attract the user, and can be represented as
    \begin{align}
        \chi(\gR_t,k) = \prod_{s=1}^{k-1} (1 - \alpha(\ba_{s}^t)).
    \end{align}
The maximum number of clicks is at most $1$, and the expected number of clicks in each round can be written as
\begin{align}\label{3}
\begin{split}
   \sum^{K}_{k = 1} v(\gR_t,\ba_{k}^t,k) = \sum_{k = 1}^K\chi(\gR_t, k)\alpha(\ba_{k}^t) = 1 - \prod_{k=1}^K (1 - \alpha(\ba_{k}^t)).
    \end{split}
\end{align}
Similar to the position-based model, $ \chi(\gR_t,1) > ... >  \chi(\gR_t,K)$ is hold in the cascade model.

\begin{definition} [No-regret ranker] \label{definition2}
We define the no-regret ranker as a ranker that achieves a sublinear ($o(T)$) regret in its click model under Assumption \ref{supassumption}. By Definition \ref{supdefinition}, we can see that a ranker is no-regret if and only if it presents $\bR^*$ to the user for $T-o(T)$ times.
\end{definition}

\begin{remark} \label{remark1}
We now briefly discuss correlations between click models and no-regret rankers. 
Recall the definition of the position-based model, the optimal list $\bR^*$ can uniquely maximize (\ref{PBM}). Thus, every ranker that achieves regret $R(T) = o(T)$ in the position-based model falls into the category of no-regret ranker (such as PBM-UCB \citep{Lagre2016MultiplePlayBI}). Besides, in the click model presented by \citet{zoghi2017onlineLT,Lattimore2018TopRankAP}, a ranker can achieve a sublinear regret if and only if they can present the optimal list $\bR^*$ for $T-o(T)$ times. Therefore, their click models also satisfy Assumption \ref{supdefinition}, and state-of-the-art online ranking methods BatchRank \citep{zoghi2017onlineLT} and TopRank \citep{Lattimore2018TopRankAP} fall into the category of no-regret rankers. However, every permutation of the first $K$-most attractive items can maximize (\ref{3}) in the cascade model.
The item with the highest attractiveness may not be placed at the first position for $T-o(T)$ times by an online stochastic ranker with $R(T) = o(T)$. Thus not all rankers that achieve $R(T) = o(T)$ in the cascade model are no-regret rankers. 
\end{remark}

\subsection{Threat models}
Let $\bbN_T(a_k)$ denote the total rounds item $a_k$ placed at the first position of $\gR_t$ until time $T$. The adversary aims to fool the ranker to place a target item $\tilde{a}$ at the first position of $\gR_t$ for $T-o(T)$ rounds. We consider two poisoning attack models.

\begin{figure}
\vspace{-5mm}
	\centering 	
 \subfigbottomskip=4pt
	\subfigcapskip=-5pt 
	\subfigure[Click poisoning attack.]{
\includegraphics[width=0.4\linewidth]{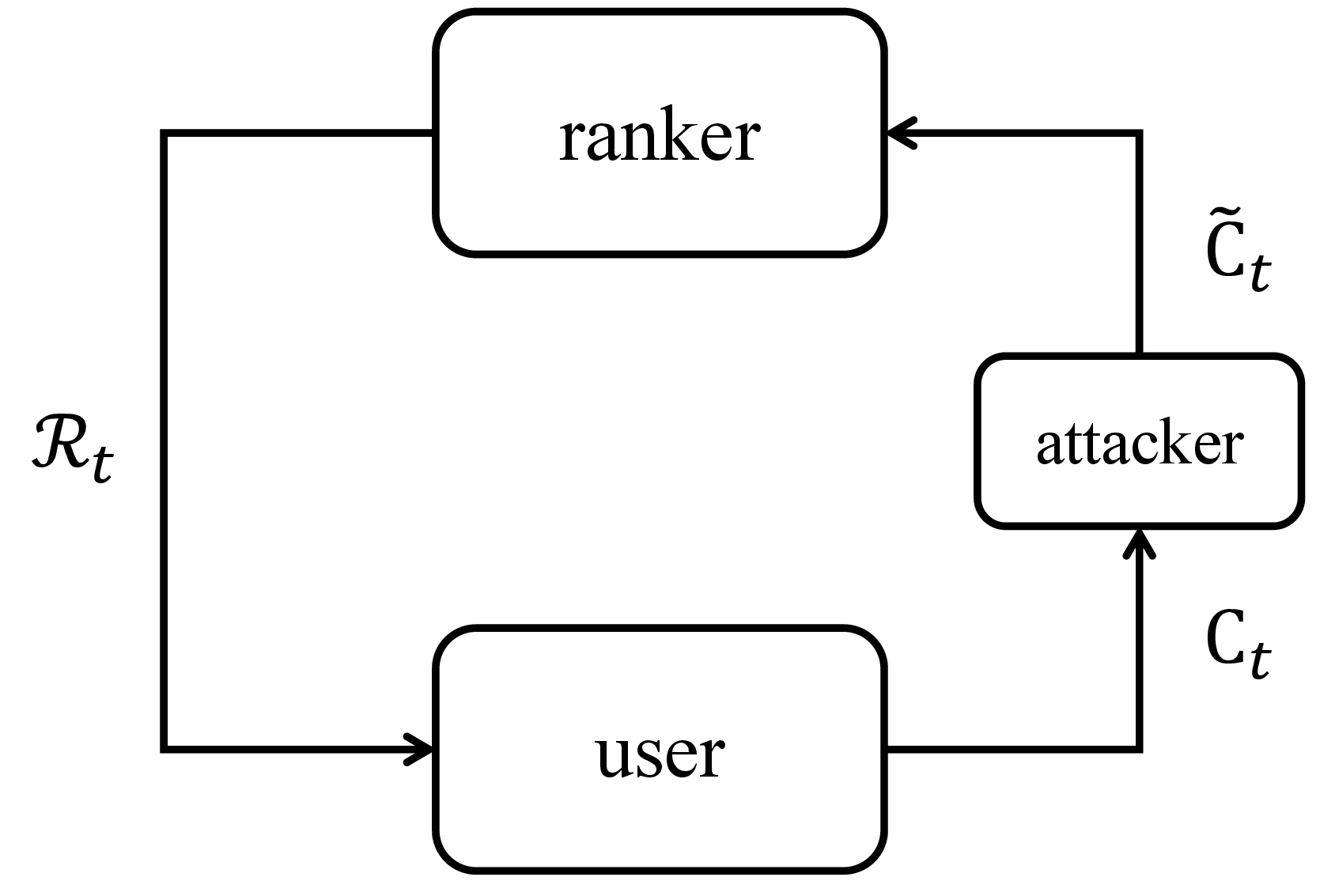}\label{fig1a}}\quad\quad\quad\quad
	\subfigure[List poisoning attack.]{
\includegraphics[width=0.4\linewidth]{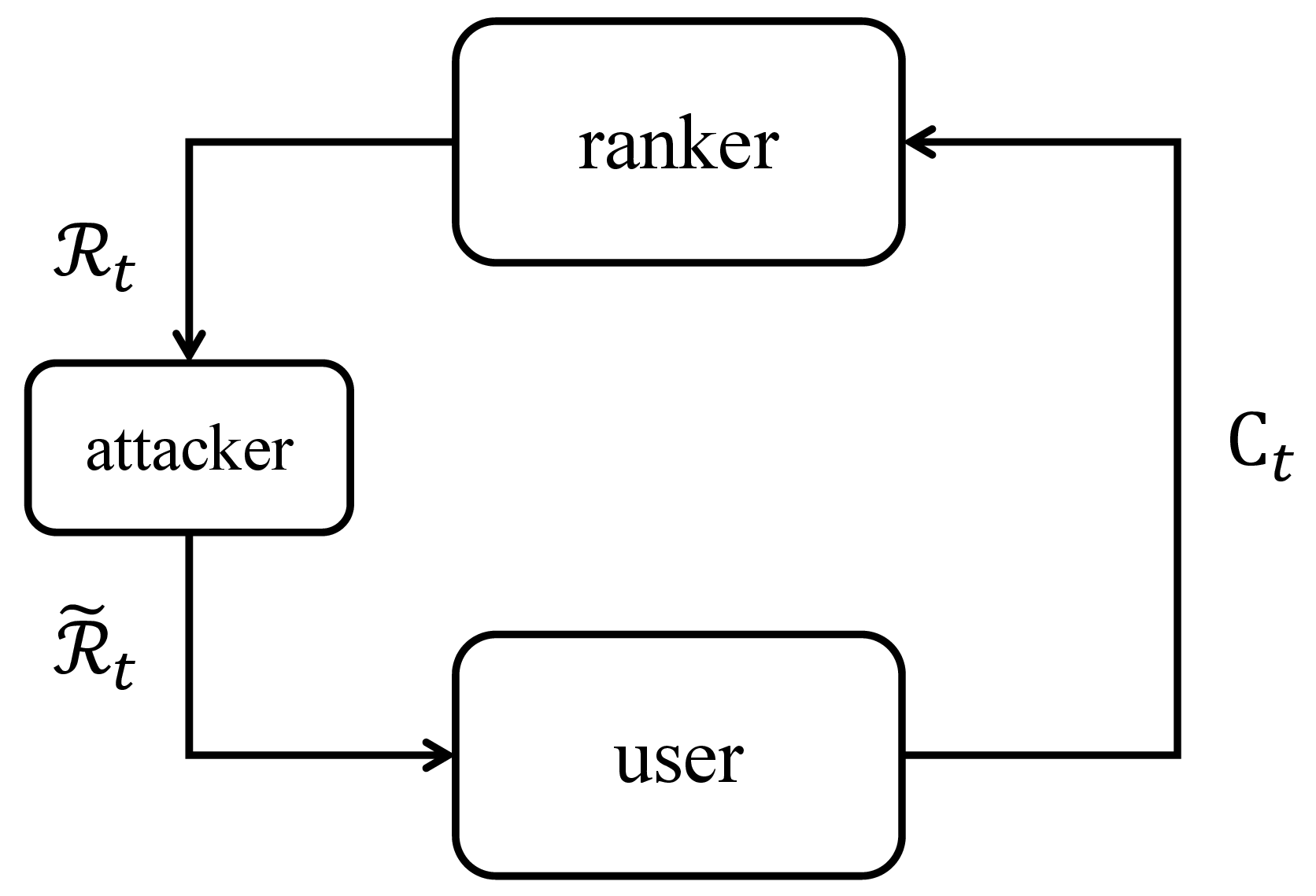}\label{fig1b}}
\caption{Threat models on online learning to rank.}
\vspace{-5mm}
\end{figure}

\paragraph{Click poisoning attacks.} 
We illustrate click poisoning attacks in Figure \ref{fig1a}. This is similar to the reward poisoning attacks studied on multi-armed bandits \citep{jun2018adversarial,liu2019data}.
In each round, the attacker obtains the user's feedback $\gC_t$, and modifies it to perturbed clicks $\tilde{\gC}_t = (\tilde{\gC}_1^t,...,\tilde{\gC}_L^t)$. Naturally, the attacker needs to attain its attack goal with minimum attack cost defined as $\bbC(T) = \sum_{t=1}^T\sum_{k=1}^L \vert\tilde{\gC}_k^t - \gC_k^t\vert$. 

\paragraph{List poisoning attack.}
Instead of directly manipulating the click feedback, the list poisoning attacks manipulate the presented ranking list from $\gR_t$ to $\tilde{\gR}_t$ as illustrated in Figure \ref{fig1b}. This is similar to the action poisoning attack proposed by \citet{Liu2020ActionManipulationAO,Liu2021EfficientAP} against multi-armed bandits. We assume the attacker can access items with low attractiveness denoted as $\{\eta_k\}_{k=1}^{2K+1}\not\in\bD$ and for convenience, $\alpha(\eta_1)>,...,>\alpha(\eta_{2K-1})$. The low attractiveness items satisfy $\alpha(\eta_1) < \alpha(a_L)$. We suppose the attacker does not need to know the actual attractiveness of these items, but only their relative utilities, i.e., the attractiveness of items in $\{\eta_k\}_{k=1}^{K-1}$ is larger than items in $\{\eta_k\}_{k=K}^{2K-1}$. The attacker uploads these items to the candidate action set before exploration and we denote $\tilde{\bD} = \bD\cup\{\eta_k\}_{k=1}^{2K+1}$. In each round, the attacker can replace items in original ranking $\gR_t$ with items in $\{\eta_k\}_{k=1}^{2K-1}$. This modified list $\tilde{\gR}_t = (\tilde{\ba}^t_1,...,\tilde{\ba}^t_K)$ is then sent to the user. The cost of the attack is  $\bbC(T) = \sum_{t=1}^T\sum_{k=1}^K \bone\{\tilde{\ba}^t_k  \neq {\ba}^t_k \}$. Note that the click feedback $\bC_t$ in list poisoning attacks is generated by $\tilde{\bR}_t$ instead of $\bR_t$, but the ranker assumes that the feedback is for $\bR_t$.

In practice, the click poisoning attack could be related to fake clicks/click farms as mentioned in \cite{WSJ,BuzzFeed,Golrezaei2021LearningPR}; list poisoning attack could be achieved by malware installed as a browser extension, where the malware does not directly change the click feedback but can manipulate the web page of ranking list locally. We aim to design \textit{efficient attack strategies} against online rankers, which is defined as follows.

\begin{definition} [Efficient attack] \label{definition1}  We say an attack strategy is efficient if 
\begin{enumerate}
  \item It misguides an  online stochastic ranker to place the target item $\tilde{a}$ at the first position of $\gR_t$ for $T - o(T)$ times in expectation with cost $\bbC(T) = o(T)$. 
  \item To keep the click poisoning attack stealthy, the returned total clicks $\sum_{k=1}^L\tilde{\gC}_k^t$ in the cascade model is at most $1$ and in the position-based model is at most $K$.
\end{enumerate}
\end{definition}

We conclude the preliminary with the difference between poisoning attacks on stochastic bandits \citep{jun2018adversarial,liu2019data, Xu2021ObservationFreeAO} and online learning rankers. Data poisoning attack on stochastic bandits aims to fool the bandit algorithm to pull the target arm $T - o(T)$ times with $o(T)$ cost. The main idea of this class of attack strategies is to hold the expected reward of the target item and reduce the expected reward of the non-target items. However, in the OLTR setting, 1) the ranker would interact with a length $K$ list $\gR_t$ instead of a single arm; 2) \emph{the user would generate click feedback under different click models that depend on examination probability.} Recall from the definition of click models, in the position-based model the user would return at most $K$ clicks in one round, while in the cascade model, the user would return at most $1$ click. Thus, if the attacker returns more than one click in the cascade model, its attack is unstealthy and inefficient.

\section{Generalized list poisoning attack strategy}

In this section, we would propose a generalized list poisoning attack (\texttt{\texttt{GA}}) that misguides \emph{any no-regret ranker} to place the target item at the first position of $\bR_t$  for $T-o(T)$ times in expectation with $o(T)$ cost. 

 \begin{algorithm}[t]
\renewcommand{\algorithmicrequire}{\textbf{Input:}}
\renewcommand{\algorithmicensure}{\textbf{Output:}}
        \label{algorithm1}
	\caption{Generalized List Poisoning Attack (\texttt{GA}) }
	\label{alg1}
	\begin{algorithmic}[1]
            \STATE \textbf{Inputs:} List $\bT = (\tilde{a},\eta_{1},...,\eta_{K-1})$ and $\{\eta_k\}_{k=1}^{2K-1}$ 
            \STATE Upload $\{\eta_k\}_{k=1}^{2K-1}$ to the candidate action set 
            \STATE \textbf{for} $t = 1:T$ \textbf{do}
               \STATE \quad Observe $\bR_t = (\ba_{1}^t,...,\ba_{K}^t)$
               \STATE \quad \textbf{if} $\bR_t\backslash\bT \not= \emptyset$ \textbf{then}
               \STATE \quad \quad \textbf{for} $k = 1 : K$ \textbf{do}
               \STATE \quad \quad \quad \textbf{if} $\ba_k^t\not \in \bT$ \textbf{then}
               \STATE \quad \quad \quad \quad Set $\tilde{\ba}_{k}^t = \eta_{K+k-1}$.
               \STATE \quad \quad \quad \textbf{else} 
               \STATE \quad\quad\quad \quad Set $\tilde{\ba}_{k}^t = \ba_{k}^t$
               \STATE \quad \quad  Return $\tilde{\bR}_t = (\tilde{\ba}_1^t,...,\tilde{\ba}_K^t)$ to the user
               \STATE \quad \textbf{else}
               \STATE \quad\quad Do not attack
	\end{algorithmic}  
\end{algorithm}

\paragraph{\texttt{GA} against no-regret rankers.}
 We briefly illustrate the process of \texttt{GA}. The strategy is summarized in Algorithm \ref{alg1}. The attacker first needs to design list $\bT$, where $\bT = \{\tilde{a},\eta_{1},...,\eta_{K-1}\}$. Intuitively, items $\eta_1,...,\eta_{K-1}$ are utilized to fill in the rest of the positions when $\tilde{a}$ is in $\bR_t$. The attacker would not manipulate item $\ba_k^t \in \bT$. If items $\ba_k^t \in \tilde{\bD}/\bT$ are placed in the position $k$, the attacker would replace it with $\eta_{K+k-1}$ (lines 5-9 in Algorithm \ref{alg1}). This manipulation strategy can mislead the ranker to believe the items in $\tilde{D}/\bT$ are smaller than $\alpha(\eta_{K})$ (due to $\alpha(\eta_{K}) < \alpha(\eta_{K-1}) < \alpha(a_L)$). Hence, list $\bT$ would be deemed as the optimal list and $\tilde{a}$ is the item with the highest attractiveness. Theorem \ref{proposition1} demonstrates that \texttt{GA} can efficiently attack any no-regret ranker.

\begin{theorem} \label{proposition1} 
\texttt{GA} can efficiently attack any no-regret ranker, i.e., $\bE[\bbN_T(\tilde{a})] = T - o(T)$ and $\bbC(T) = o(T)$.
\end{theorem}

 Although we do not include all rankers with $o(T)$ regret in the cascade model as no-regret rankers from Remark \ref{remark1}, we can still show that classic cascade ranking algorithms such as CascadeUCB1 \citep{Kveton2015CascadingBL} can be efficiently attacked by Algorithm \ref{alg1}.

\begin{theorem} \label{suptheorem} 
Consider the ranker applies CascadeUCB1 and is attacked by Algorithm \ref{alg1}. The expected number of $\tilde{a}$ be placed at the first position till round $T$ satisfies 
\begin{align}
\begin{split}
\bE[\bbN_T(\tilde{a})] \ge T - \sum_{k=1}^{K-1}\frac{ 3 + 81\log(T)}{(\alpha(\tilde{a}) - \alpha(\eta_k))^2} - \sum_{k=1}^{L+K-1}\frac{ 3 + 81\log(T)}{(\alpha(\tilde{a})-\alpha(\eta_K))^2}
\end{split}
\end{align}
and the cost is $o(T)$.
\end{theorem}

After the above observation, one may wonder why we don't utilize click poisoning strategy to achieve the same goal of \texttt{GA}, we propose an motivated example.

\paragraph{Example 1.}
    Consider an example of the Cascade model, where the examination probability of an item in $\bR_t$
 is related to other items' attractiveness in $\bR_t$. We suppose the case when item $a_1\in\bT$
 is placed before item $a_2\not\in\bT$
 and the click feedback of item $a_1$ is $1$ (which implies the user will not examine the following items and thus true click feedback of $a_2$
 will be $0$). If the attacker trivially reduces the click feedback of all the items $\not\in\bT$
 to $0$ (which is a common strategy of attack on bandits \citep{jun2018adversarial,garcelon2020adversarial}), this can be interpreted as the attractiveness of item $a_2$
 is reduced to $0$. Since $a_1$
 is not clicked, the following items should be examined and the OLTR algorithm would recognize items placed after $a_1$
 (includes $a_2$) as $0$
 attractiveness. The click manipulation strategy clearly harms the attack in this cascade model example, making the attack results hard to be analyzed. According to this instance, existing reward (e.g., click) poisoning strategies on bandits can hard to be proved to succeed in different click models, as the clicks should be manipulated according to the property of the click model. However, our \texttt{GA} can adapt to stochastic click models for \emph{any no-regret ranker} and enjoys a simple theoretical characteristic.

\begin{remark}
The idea of \texttt{GA} against online stochastic rankers is similar to the previous reward poisoning attack idea against stochastic bandits, i.e., reduces the expected reward (i.e., clicks) of the non-target items and holds the expected reward of the target item. The main difference is 1) we enlarge our target from an item to a list; 2) we manipulate the ranking list $\bR_t$ instead of manipulating 'rewards'; 3) the attack is applied to the comprehensive stochastic click models. 
\end{remark}

\section{Attack-then-quit strategy}
    In this section, we provide a click poisoning attack strategy that applied to elimination-based rankers and divide-and-conquer-based rankers. We will demonstrate our strategy can \emph{efficiently attack} some specific rankers under any click models. Previous reward poisoning attack strategies on multi-armed bandits only manipulate the reward of the non-target item and never perturb the reward of the target item \citep{jun2018adversarial,liu2019data,garcelon2020adversarial,wang2022linear}. However, the attack strategy proposed in this section perturbs the click feedback of every item. The attacker would spend a linear cost in a short period $T_1 = o(T)$ and then stop the attack and quit. Thus, we call our attack strategy the attack-then-quit (\texttt{ATQ}) strategy. 
    
    As mentioned in the previous section, the BatchRank and TopRank can work in general stochastic click models, while other UCB-type algorithms can only work in specific click models (e.g., CascadeUCB1 or PBM-UCB). Hence, the BatchRank and TopRank can be deemed as the superior ranker to those UCB-type algorithms, which are worth more significant consideration. Moreover, to our best knowledge, current algorithms dedicate to the stochastic click models all follow the elimination/divide-and-conquer idea.

\subsection{Click poisoning-based \texttt{ATQ}}
Algorithm \ref{alg2} can efficiently attack elimination-based rankers (e.g., BatchRank) and divide-and-conquer-based rankers (e.g., TopRank). At the initialization stage, the attacker would receive target item $\tilde{a}$ and threshold $T_1$. In each round, if click feedback $\bc_{k}^t$ is returned from $\tilde{a}$ and $\tilde{a}\in \bR_t$, the attacker would set $\tilde{\bc}_{k}^t = 1$, otherwise, $\tilde{\bc}_{k}^t = 0$.  Besides, as we mention in Assumption 1 and Definition 1, the item $a_k\not\in\bR_t$ would not be observed by the user in round $t$, and would not be clicked either. This implies their click feedback would be 0. In ATQ, the attacker does not (cannot) modify the click feedback of items $a_k\not\in\bR_t$ 
 (line 9$\sim$10 in Algorithm 2). After scanning $\bC_t$, the attacker would return $\tilde{\bC}_t = (\tilde{\bc}_{1}^t,...,\tilde{\bc}_{L}^t)$ to the ranker. Note that $\sum_{k=1}^L \tilde{\bc}_k^t \le 1$ (line 4-10 in Algorithm \ref{alg2}), hence Algorithm \ref{alg2} can work stealthily and efficiently in different click models. Note that the general stochastic click models in \citep{zoghi2017onlineLT} and \citep{Lattimore2018TopRankAP} directly ensure $\bR^*$ can uniquely maximize the expected click number. Hence, the target of BatchRank and TopRank is to quickly estimate the attractiveness of items and rank them accordingly. Our attack Algorithm \ref{alg2} aims to mislead these rankers to 'believe' the target item is the best item with the smallest rounds. Once these rankers make this judgment, they would place the target item at the first position until $T$. In the following sections, we will demonstrate Algorithm \ref{alg2} can efficiently attack BatchRank and TopRank with suitable $T_1$.

 \begin{algorithm}[t]
\renewcommand{\algorithmicrequire}{\textbf{Input:}}
\renewcommand{\algorithmicensure}{\textbf{Output:}}
	\caption{Attack-Then-Quit (\texttt{ATQ}) Algorithm}
	\label{alg2}
	\begin{algorithmic}[1]
            \STATE \textbf{Inputs:} The target item $\tilde{a}$, threshold $T_1$
              \STATE \textbf{for} $t = 1 : T_1$ \textbf{do}
              \STATE \quad Observe $\bC_t = (\bc_{1}^t,...,\bc_{L}^t)$ and $\bR_t = (\ba_{1}^t,...,\ba_{K}^t)$
              \STATE \quad \textbf{for} $k = 1:L$ \textbf{do}
              \STATE  \quad\quad \textbf{if} $a_k = \tilde{a}$ and $a_k \in \bR_t$ \textbf{then}
              \STATE \quad \quad \quad Set $\tilde{\bc}_{k}^t = 1$
              \STATE \quad \quad \textbf{else if} $a_k \not = \tilde{a}$ and $a_k \in \bR_t$ \textbf{then}
              \STATE \quad \quad \quad Set $\tilde{\bc}_{k}^t = 0$
              \STATE \quad \quad \textbf{else}
              \STATE \quad \quad \quad Set $\tilde{\bc}_{k}^t = \bc_k^t$
              \STATE \quad Return $\tilde{\bC}_t = (\tilde{\bc}_1^t,...,\tilde{\bc}_L^t)$ to the ranker
	\end{algorithmic}  
\end{algorithm}

\subsection{Attack on BatchRank}

\paragraph{Brief explanation of BatchRank.}
 The BatchRank falls into the category of elimination-based algorithms \citep{zoghi2017onlineLT}. The BatchRank would begin with stage $\ell_1 = 0$ and the first batch $B_{1,\ell_1} = \bD$. In stage $\ell_1$, every item would be explored for $\bn_{\ell_1} = 16\tilde{\Delta}^{-1}_{\ell_1}\log(T)$ times and $\tilde{\Delta}^{-1}_{\ell_1} = 2^{-\ell_{1}}$. Afterward, if BatchRank collects enough information to confirm the attractiveness of every item in a group is larger than the others, BatchRank then splits the initial batch into two sub-batches $B_{2,\ell_{2}}$ and $B_{3,\ell_{3}}$.  Items that are considered with higher attractiveness (i.e., in sub-batch $B_{2,\ell_2}$) would be always placed before items with lower attractiveness (i.e., in sub-batch $B_{3,\ell_3}$). The BatchRank would restart with stage $\ell_{2} = 0$ and $\ell_{3} = 0$ and sub-batches $B_{2,\ell_{2}}$ and $B_{3,\ell_{3}}$. Batches would recursively split until round $T$.
 Intuitively, the action `split' of BatchRank is similar to the elimination action in the elimination-based bandit algorithms \citep{EvenDar2006ActionEA,lykouris2018stochastic,bogunovic2021stochastic}. The  details of BatchRank are provided in the appendix. 
 
  Algorithm \ref{alg2} can successfully attack BatchRank owing to BatchRank's elimination property. Algorithm \ref{alg2} maximizes the returned clicks of the target item and minimizes the returned click of the non-target item in a short period $o(T)$. After this period, BatchRank regards the target item owning the highest attractiveness (i.e., split). Then the target item would be always placed at the first position of the list $\bR_t$ and the attacker achieves the attack goal. We also provide an example in the appendix to further illustrate the process of Algorithm \ref{alg2} attack BatchRank.

\begin{theorem} [Attack against BatchRank] \label{theorem1} Consider the attacker aims to attack BatchRank with stochastic click models  \citep{zoghi2017onlineLT}. If the attacker runs Algorithm \ref{alg2} with $T_1 = 16L\log(T)$, we have $\bbN_T(\tilde{a}) \ge T - T_1$ and the cost can be bounded by $\bbC(T) \le KT_1$.
\end{theorem}

 Besides BatchRank, this attack idea can also be utilized to attack some rankers that do not belong to the elimination-based category, such as TopRank.

\subsection{Attack on TopRank}

\paragraph{Brief explanation of TopRank.}
TopRank is a divide-and-conquer-based ranker \citep{Lattimore2018TopRankAP}. It  begins with a blank graph $G_1 = \emptyset$. In round $t$, TopRank would establish blocks $\bP_{t1},...,\bP_{td}$ via graph $G_t$. The items in block $\bP_{t1}$ would be placed at the first $\vert \bP_{t1} \vert$ positions and the items in block $\bP_{t2}$ would be placed at the next $\vert \bP_{t2} \vert$ positions, and so on. During rounds $1$ to $T$, TopRank would explore items with blocks, collect click information and compare attractiveness between items in the same block. If the collected evidence is enough to let TopRank regards the attractiveness of item $a_i$ as larger than the attractiveness of item $a_j$, a directional edge $(a_j,a_i)$ would be established. This behavior is similar to the `split' action in BatchRank. Besides, graph $G_t$ would not contain cycles with high probability. If the graph contains at least one cycle, we consider TopRank would be out of control. Details of TopRank are provided in the appendix. 

Note that if there exist edges from every non-target item to the target item and $G_t$ contains no cycle, then the target item would be isolated from the non-target items and would always be placed at the first position of $\bR_t$. This is because the first block only contains the target item. We also provide an example to specifically explain how Algorithm \ref{alg2} attacks TopRank in the appendix.

\begin{theorem} [Attack against TopRank] \label{theorem2}
Consider the attacker aims to attack TopRank and TopRank runs in the click model in \citep{Lattimore2018TopRankAP}. If the attacker sets $T_1 = \frac{4\log(c/\delta)}{\frac{K}{L} + (1 - \sqrt{1 + 8K/L})/4}$ and runs Algorithm \ref{alg2}. The attacker can achieve $\bbN_T(\tilde{a}) \ge T - T_1$ with probability at least $1-(1/c+L^2)\delta$. The cost can be bounded by $\bbC(T) \le KT_1$.
\end{theorem}

By choosing $\delta = 1/T$ and $c = 4\sqrt{2/\pi}/\emph{erf}(\sqrt{2}) \approx 3.43$ which is same as in TopRank algorithm, we have $T_1 = O((L/K)\log T)$. The proof of Theorem \ref{theorem2} mainly focuses on how to bound the number of the target item to be placed in $\bR_t$ ($\sum_{s=1}^t\bone\{\tilde{a}\in\bR_s\}$) when $G_t = \emptyset$. Note that we can manipulate the click of the target item only if $\tilde{a} \in \bR_t$. Hence, we can deduce when are the edges from the non-target item to the target item established with $\sum_{s=1}^t\bone\{\tilde{a}\in\bR_s\}$. The probability of the attack failure is at most $L^2\delta + \delta/c$, where $L^2\delta$ is the intrinsic probability of TopRank's $G_t$ contains cycle and $\delta/c$ is the probability the attacker fails to bound $\sum_{s=1}^t\bone\{\tilde{a}\in\bR_s\}$ when $G_t(\tilde{a}) = \emptyset$.


\section{Experiments}
\begin{figure*}[t]
	\centering  
        \subfigbottomskip=0pt 
	\subfigcapskip=-5pt 
	\subfigure[Cost in CM.]{
\includegraphics[width=0.33\linewidth]{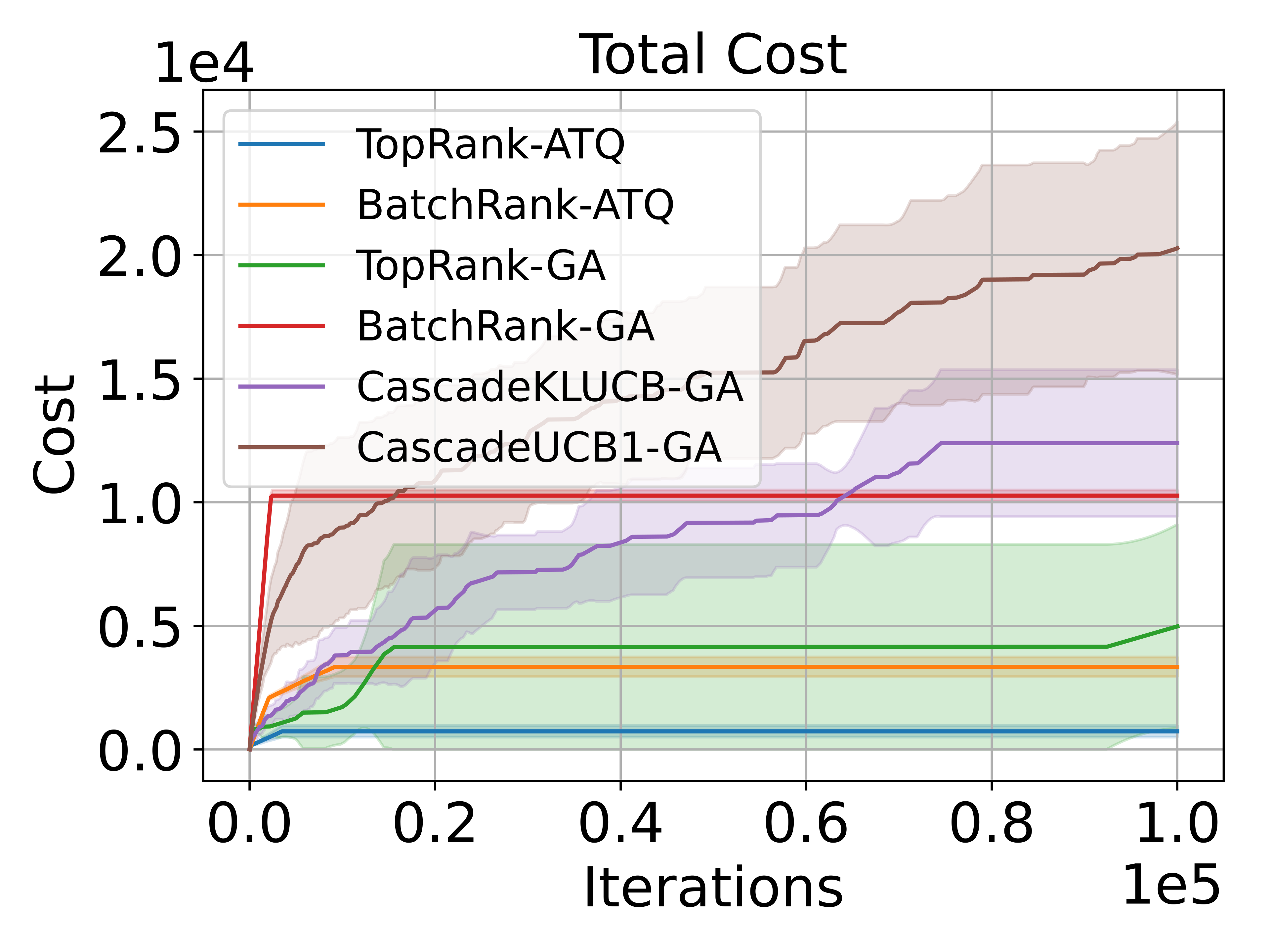}\label{fig:CM_synthetic_cost}}\quad\quad\quad\quad
	\subfigure[$\bbN_t(\tilde{a})$ in CM.]{
\includegraphics[width=0.33\linewidth]{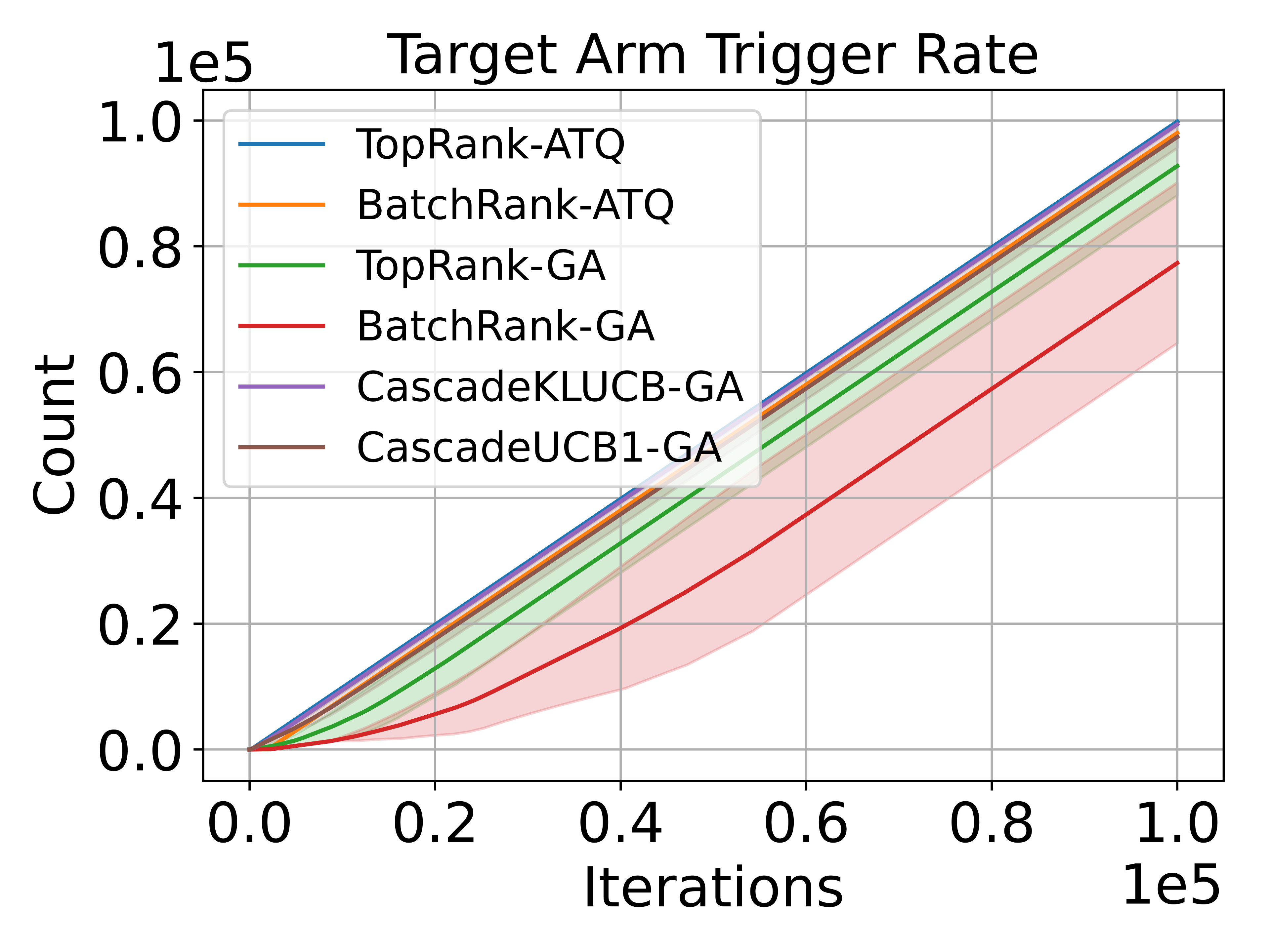}\label{fig:CM_synthetic_rate}}\\
\subfigure[Cost in PBM.]{
\includegraphics[width=0.33\linewidth]{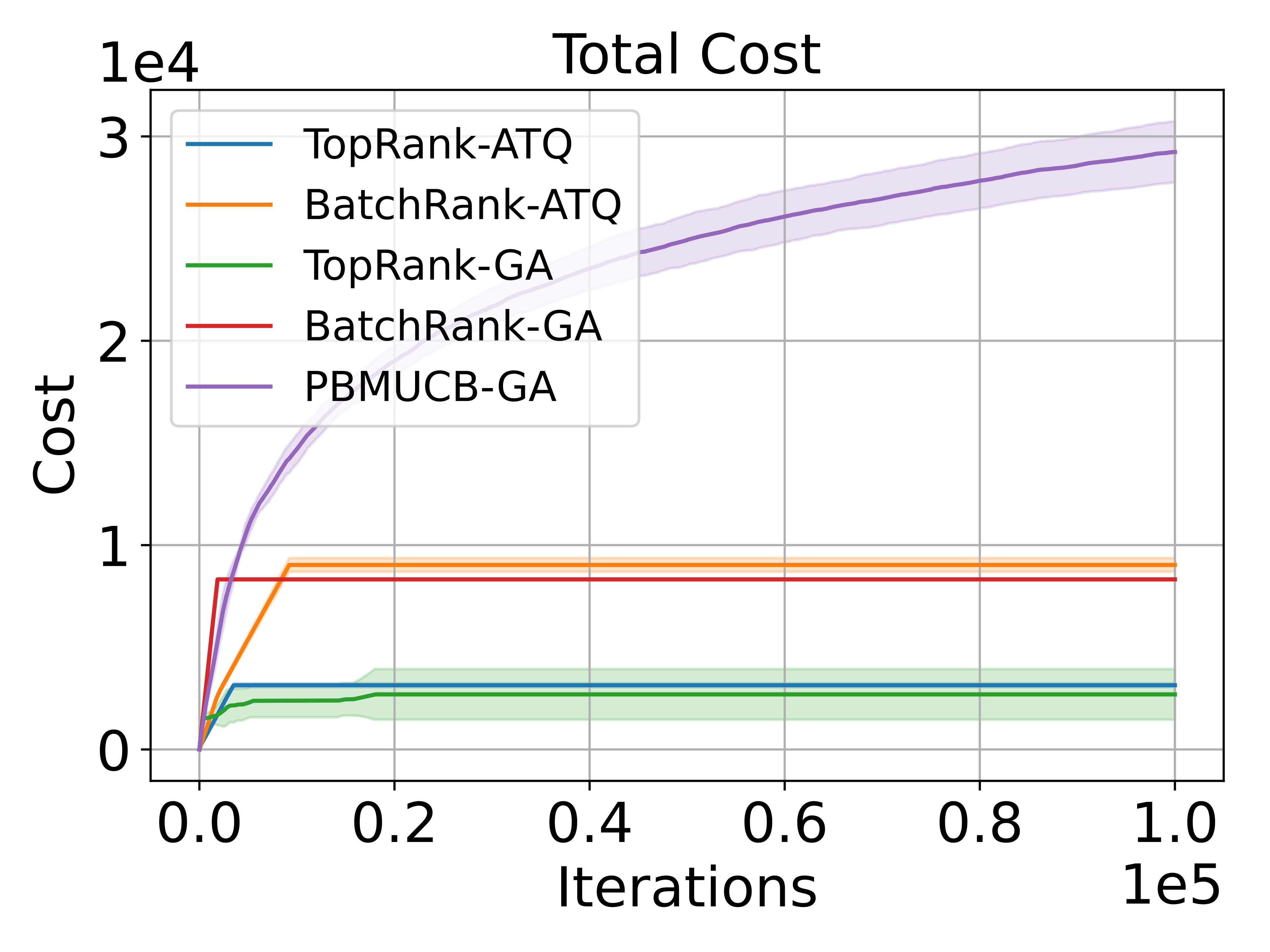}\label{fig:PBM_synthetic_cost}}\quad\quad\quad\quad
\subfigure[$\bbN_t(\tilde{a})$ in PBM.]{
\includegraphics[width=0.33\linewidth]{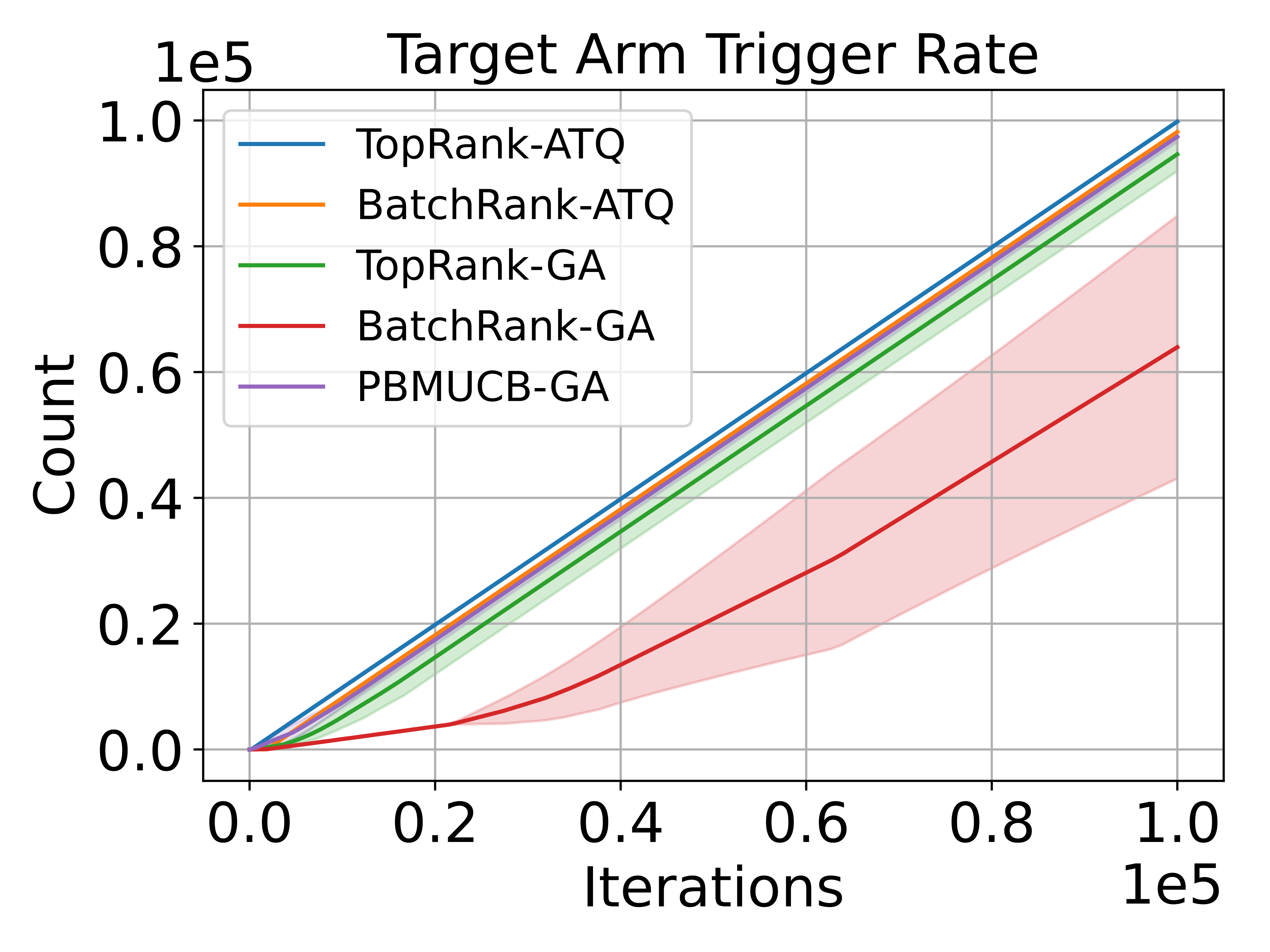}\label{fig:PBM_synthetic_rate}}
\caption{Synthetic data experiment: (a) total cost spend in the cascade model, (b) $\bbN_t(\tilde{a})$ in the cascade model, (c) the total cost spend in the cascade model and (d) $\bbN_t(\tilde{a})$ in the position-based model. We report averaged result
and variance of 10 runs.}
\label{fig:synthetic_results}
\end{figure*}

In the experiment section, we apply the proposed attack methods against the OLTR algorithms listed in Table \ref{tab:simulation results} with their corresponding click models. We compare the effectiveness of our attack on synthetic data and real-world MovieLens dataset. For all our experiments, we use $L=50$, $K=5$ (the set up of \citet{zoghi2017onlineLT,Lattimore2018TopRankAP} is $L=10$ and $K=5$) and $T = 10^5$. For ATQ, we set the $T_1$ in Algorithm \ref{alg2} by Theorem \ref{theorem1} and Theorem \ref{theorem2}.

\begin{table}
\centering
\caption{Target ranking algorithms and their applied click models}\label{tab:simulation results}
\begin{tabular}{ll}
   \toprule
   Algorithm & Click model\\
   \midrule
    BatchRank \cite{zoghi2017onlineLT} & Stochastic click model\\
    TopRank \cite{Lattimore2018TopRankAP} & Stochastic click model\\
    PBM-UCB \cite{Lagre2016MultiplePlayBI}& Position-based model\\
    CascadeUCB1 \cite{Kveton2015CascadingBL} & Cascade model \\
    CascadeKLUCB \cite{Kveton2015CascadingBL}& Cascade model\\
   \bottomrule
\end{tabular}
\end{table}

\subsection{Synthetic data}
First, we verify the effectiveness of our proposed attack strategies on synthetic data. We generate a size-$L$ item set $\bD$, in which each item $a_k$ is related to a unique attractiveness score $\alpha(a_k)$. Each attractiveness score  $\alpha(a_k)$ is drawn from a uniform distribution $U(0,1)$. We randomly select a suboptimal target item $\tilde{a}$. Figure \ref{fig:synthetic_results} shows the results and variances of 10 runs.

In Figures \ref{fig:CM_synthetic_cost} and \ref{fig:CM_synthetic_rate}, we plot the results of the \texttt{GA} against CascadeUCB1, CascadeKLUCB, BatchRank, and TopRank, and the \texttt{ATQ} against BatchRank and TopRank in the cascade model. Both attack strategies can efficiently misguide the rankers to place the target item at the first position for $T-o(T)$ times as shown in Figure \ref{fig:CM_synthetic_rate}, and the cost of the attack is sublinear as shown in Figure \ref{fig:CM_synthetic_cost}. The \texttt{GA} is cost-efficient when attacking all four algorithms. We can observe that when it attacks TopRank and BatchRank, the cost would not increase after some periods (similar to the \texttt{ATQ}'s results). This is when the TopRank and BatchRank believe the target item and the auxiliary items have a relatively higher attractiveness than the other items, they would only put the target item and the auxiliary items in $\bR_t$. Besides, when attacking TopRank and BatchRank, the growth rate of \texttt{GA}'s target arm pulls $N_t(\tilde{a})$ slowly increased from $0.2$ per iteration to 1 per iteration. This is because the \texttt{GA} does not manipulate the items in $\bT$ and the TopRank and BatchRank need time to confirm the target item has a higher attractiveness than $\{\eta_k\}_{k=1}^{K-1}$. Hence, the smaller the gap between $\tilde{a}$ and $\eta_1$, the larger the confirmed time. Compare with the \texttt{GA}, the \texttt{ATQ} can also efficiently attack BatchRank and TopRank with a sublinear cost. However, its $N_T(\tilde{a})$ is almost $T$, which is relatively larger than \texttt{GA}'s $N_T(\tilde{a})$. This is because the \texttt{ATQ} is specifically designed for divide-and-conquer-based algorithms like TopRank and BatchRank. The \texttt{ATQ} can maximize the target item's click number and misguide these algorithms to believe the target item is the best in the shortest period. 

Figures \ref{fig:PBM_synthetic_cost} and \ref{fig:PBM_synthetic_rate} report the results in the position-based model. We can observe that the spending cost of the \texttt{GA} on the PBM-UCB is slightly larger than the spending cost on the CascadeKLUCB and CascadeUCB1. Besides, although the \texttt{GA} can let the TopRank believe the target item is the best item in almost 500 iterations, it still needs a large number of iterations (around $6\times10^4$ iterations) to make the BatchRank make such a decision. From the results of the two models, the \texttt{ATQ} is obviously more effective than the \texttt{GA} when the target algorithms are TopRank and BatchRank.

Due to the page limitation, the experiment results based on real-world data are provided in the appendix.


\section{Related Work}

\paragraph{Online learning to rank.}
OLTR is first studied as ranked bandits \citep{Radlinski2008HowDC, Slivkins2013RankedBI}, where each position in the list is modeled as an individual multi-armed bandits problem \citep{auer2002finite}. 
Such a problem can be settled down by bandit algorithms which can maximize the expected click number in each round. Recently studied of OLTR focused on different click models \citep{Craswell2008AnEC, Chuklin2015ClickMF}, including the cascade model \citep{Kveton2015CascadingBL, Kveton2015CombinatorialCB, Zong2016CascadingBF, li2016contextual,Vial2022MinimaxRF}, the position-based model \citep{Lagre2016MultiplePlayBI} and the dependent click model \citep{Katariya2016DCMBL, Liu2018ContextualDC}. OLTR with general stochastic click models is studied in \citep{zoghi2017onlineLT, Lattimore2018TopRankAP, Li2018BubbleRankSO,Li2019OnlineLT,Gauthier2022UniRankUB}. 

\paragraph{Adversarial attack against bandits.}
Adversarial reward poisoning attacks against multi-armed bandits have been recently studied in stochastic bandits \citep{jun2018adversarial,liu2019data,Xu2021ObservationFreeAO} and linear bandits \citep{wang2022linear, garcelon2020adversarial}. These works share a similar attack idea, where the attacker holds the reward of the target arm, meanwhile lowers the reward of the non-target arm. Besides reward poisoning attacks, other threat models such as action poisoning attacks \citep{Liu2020ActionManipulationAO, Liu2021EfficientAP} were also being studied. However, adversarial attack on online ranking problem has not been explored yet. In this paper, we first time studied click poisoning attacks and list poisoning attacks against OLTR algorithms. Our click poisoning attacks share the same threat model as reward poisoning attacks, and list poisoning attacks follow a similar idea as action poisoning attacks against multi-armed bandits. 

\section{Conclusion}
 In this paper, we proposed the first study of adversarial attacks on online learning to rank. Different from the poisoning attacks studied in the multi-armed bandits setting where reward or action is manipulated, the attacker manipulates binary click feedback instead of reward and item list instead of a single action in our model. In addition, due to the interference of the click models, it is difficult for the attacker to precisely control the ranker behavior under different unknown click models with simple click manipulation. Based on this insight, we developed the \texttt{GA} that can efficiently attack any no-regret ranking algorithm. Moreover, we also proposed the \texttt{ATQ} that follows the click poisoning idea, which can efficiently attack BatchRank and TopRank. Finally, we presented experimental results based on synthetic data and real-world data that validated the cost-efficient and effectiveness of our attack strategies. 

 In our future work, it is interesting to study the adversarial attack on online learning to rank where the target is a list instead of a single item. Another intriguing direction is to establish robust rankers against poisoning attacks. In the ideal case, the robust ranker should achieve sublinear regret in general stochastic click models under different threat models. 

\bibliography{main}

\appendix
\onecolumn 
\section{Notations}
For clarity, we collect the frequently used notations in this paper.

\begin{center}
\begin{tabular}{l|l}
    $\bD$ & Total item set\\
    $\bR_t$ & $K$-length item list be shown to the user in round $t$\\
    $\bR^*$ & Optimal list\\
    $\tilde{\bR}_t$ & Manipulated list in round $t$\\
    $\bC_t$ & Click feedback list in round $t$\\
    $\tilde{\bC}_t$ & Manipulated click feedback list in round $t$\\
    $\bT$ & Ordered list $(\tilde{a},\bar{a}_1,...,\bar{a}_K)$\\
    $\tilde{a}$ & Target item \\
    $a_k$ & $k$-th most attractive item in $\bD$\\
    $\bar{a}_k$ & $k$-th most attractive auxiliary item\\
    $\eta_k$ & Particular item in the list poisoning attack\\
    $\alpha(a_k)$ & Attractiveness of item $a_k$\\
    $\ba_k^t$ & Item on the $k$-th position in $\bR_t$\\
    $\tilde{\ba}_k^t$ & Manipulated item on the $k$-th position in $\tilde{\bR}_t$\\
    $\bc_k^t$ & Click feedback of item $a_k$ in round $t$ \\
    $\tilde{\bc}_k^t$ & Manipulated click feedback of the item $a_k$ in round $t$ \\
    $v(\bR_t,\ba_k^t,k)$ & Click probability of item at the $k$-th position in round $t$\\
    $R(T)$ & Cumulative regret in $T$ rounds\\
    $\bbC(T)$ & Total cost in $T$ rounds\\
    $\bbN_t(a_k)$ & Number of item $a_k$ be placed at the first position in $t$ rounds\\
    $\bN_t(a_k)$ & Number of item $a_k$ be examined in $t$ rounds\\
    $T$ & Total number of interaction\\
    $T_1$ & Input threshold value of the attack-then-quit algorithm\\
    \textbf{BatchRank} & \\
    $b$ & Batch index\\
    $\ell$ & Stage index\\
    $B_{b,\ell}$ & $b$-th batch explored in stage $\ell$\\
    $\bn_\ell$ & Exploration number of item in batch $B_{b,\ell}$ in stage $\ell$\\
    $\bc_{b,\ell}(a_k)$ & Total received click number of item $a_k$ during stage $\ell$\\
    $\hat{\bc}_{b,\ell}(a_k)$ & Attractiveness estimator of item $a_k$ in stage $\ell$\\
    $U_{b,\ell}(a_k)$ & Upper confidence bound of item $a_k$ in stage $\ell$\\
    $L_{b,\ell}(a_k)$ & Lower confidence bound of item $a_k$ in stage $\ell$\\
    \textbf{TopRank} & \\
    $G_t$ & Auxiliary graph in round $t$\\
    $(a_j,a_i)$ & Directional edge from item $a_j$ to item $a_i$\\
    $\bP_{tc}$ & $c$-th block in round $t$\\
    $S_{tij}$ & Sum of the $U_{tij}$ from round $1$ to $t$\\
    $N_{tij}$ & Sum of the absolute value of $U_{tij}$ from round $1$ to $t$
\end{tabular}
\end{center}

\newpage

\section{Additional experiment}

\begin{figure*}
	\centering  
	\subfigbottomskip=0pt 
	\subfigcapskip=-5pt 
	\subfigure[Cost in CM.]{
\includegraphics[width=0.4\linewidth]{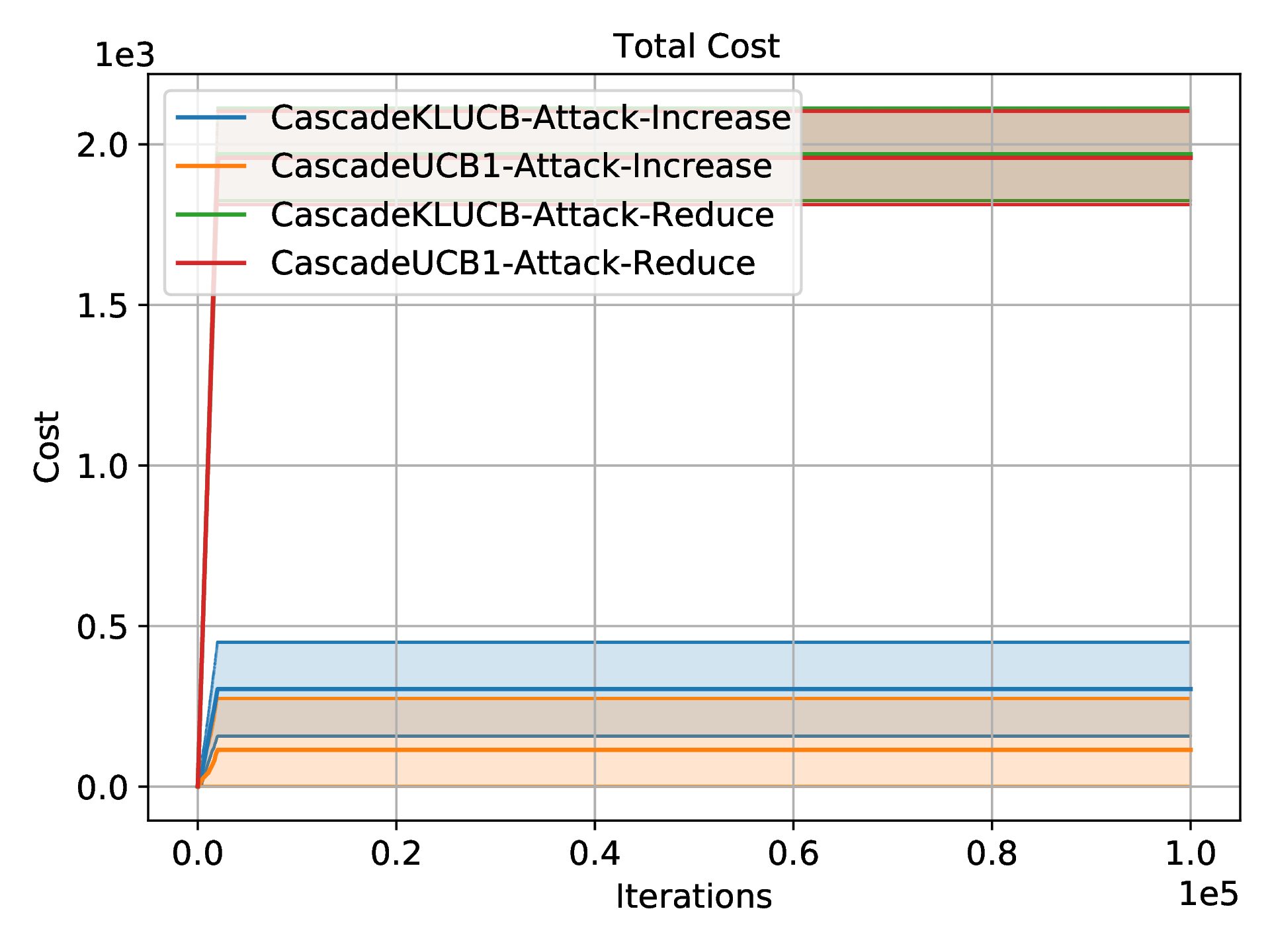}\label{fig:additional_baseline_cascade_cost}}\quad\quad\quad\quad
	\subfigure[$\bbN_t(\tilde{a})$ in CM.]{
\includegraphics[width=0.4\linewidth]{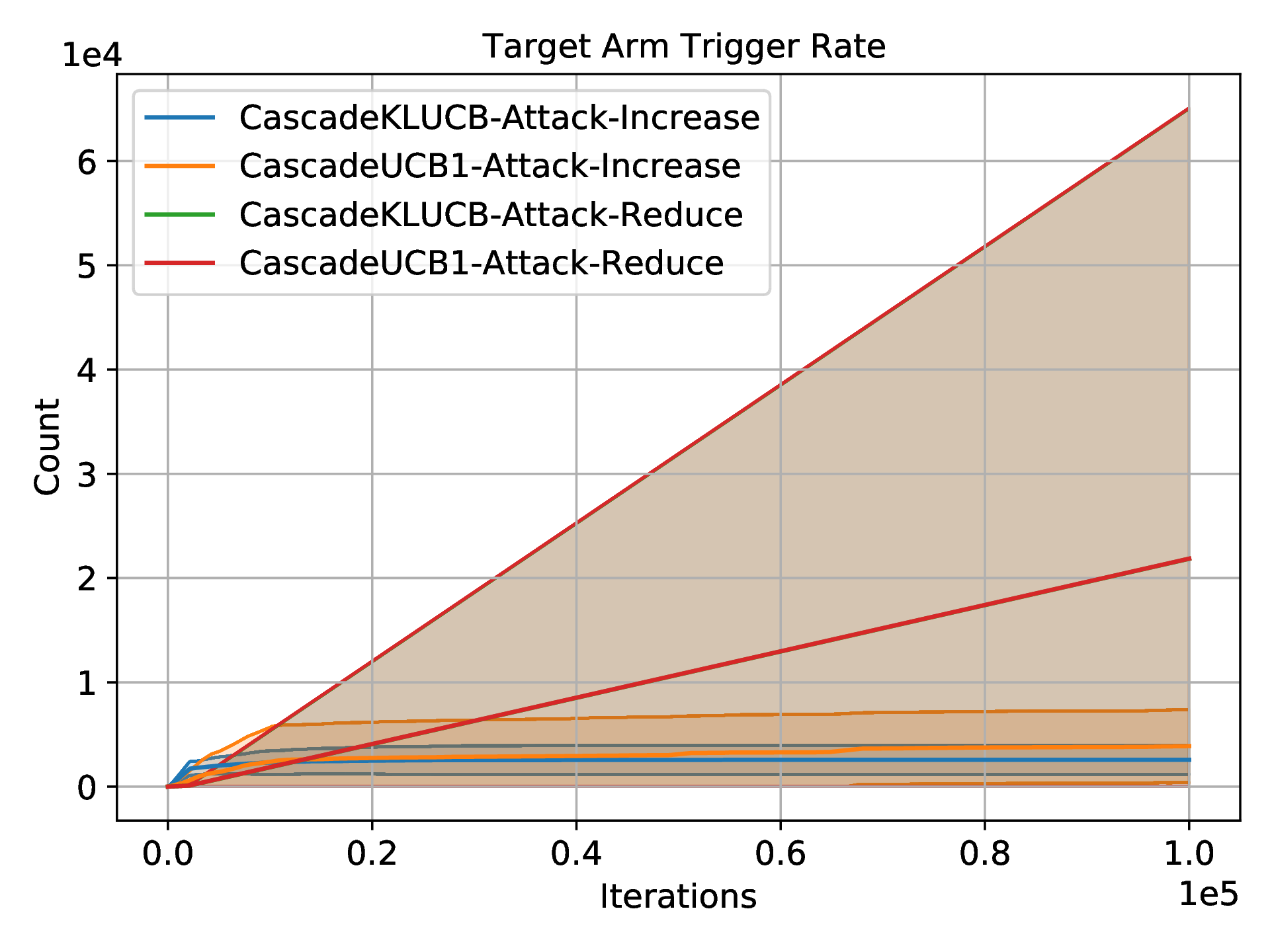}\label{fig:additional_baseline_cascade_rate}}\\
\subfigure[Cost in PBM.]{
\includegraphics[width=0.4\linewidth]{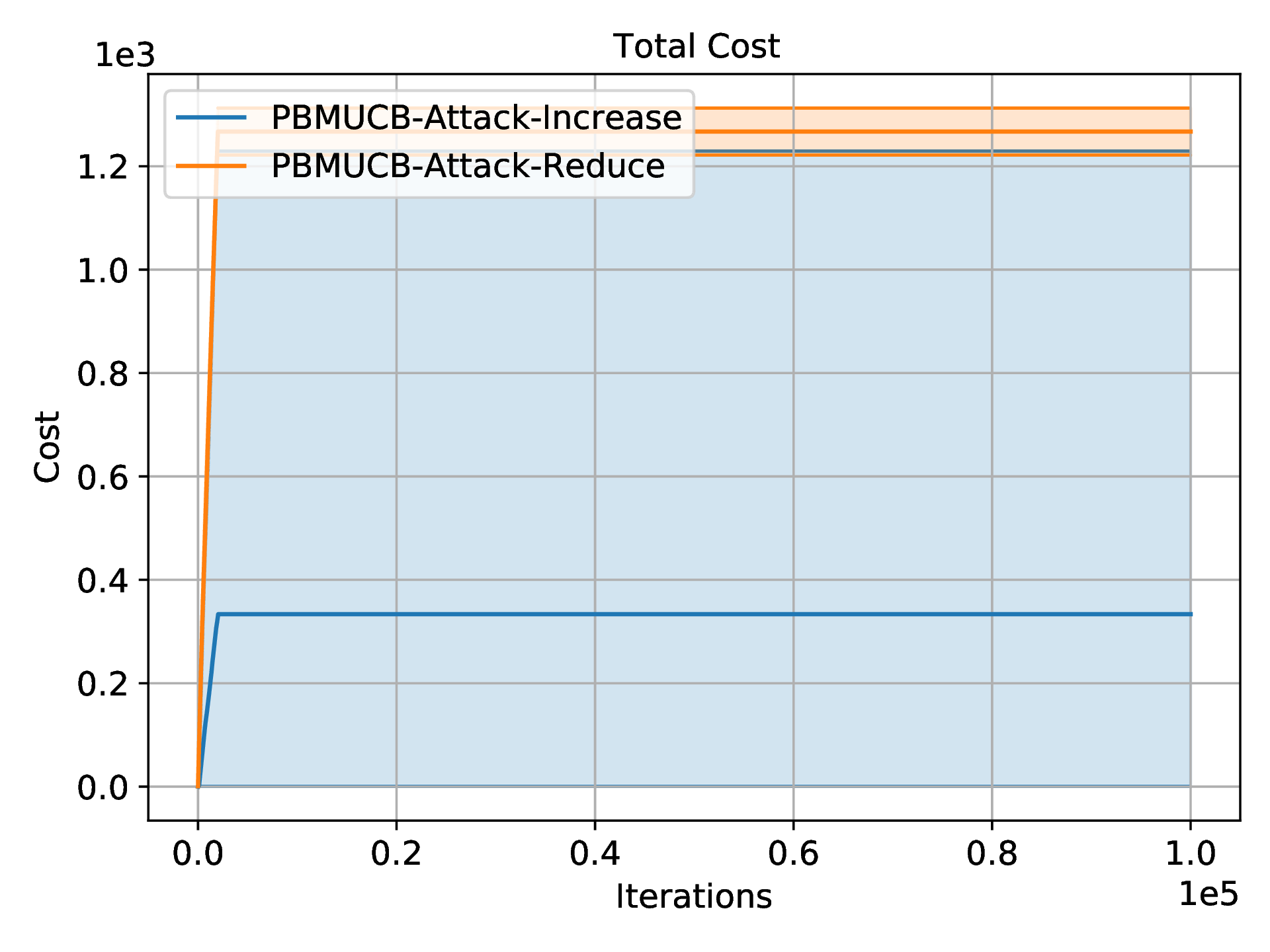}\label{fig:additional_baseline_pbm_cost}}\quad\quad\quad\quad
\subfigure[$\bbN_t(\tilde{a})$ in PBM.]{
\includegraphics[width=0.4\linewidth]{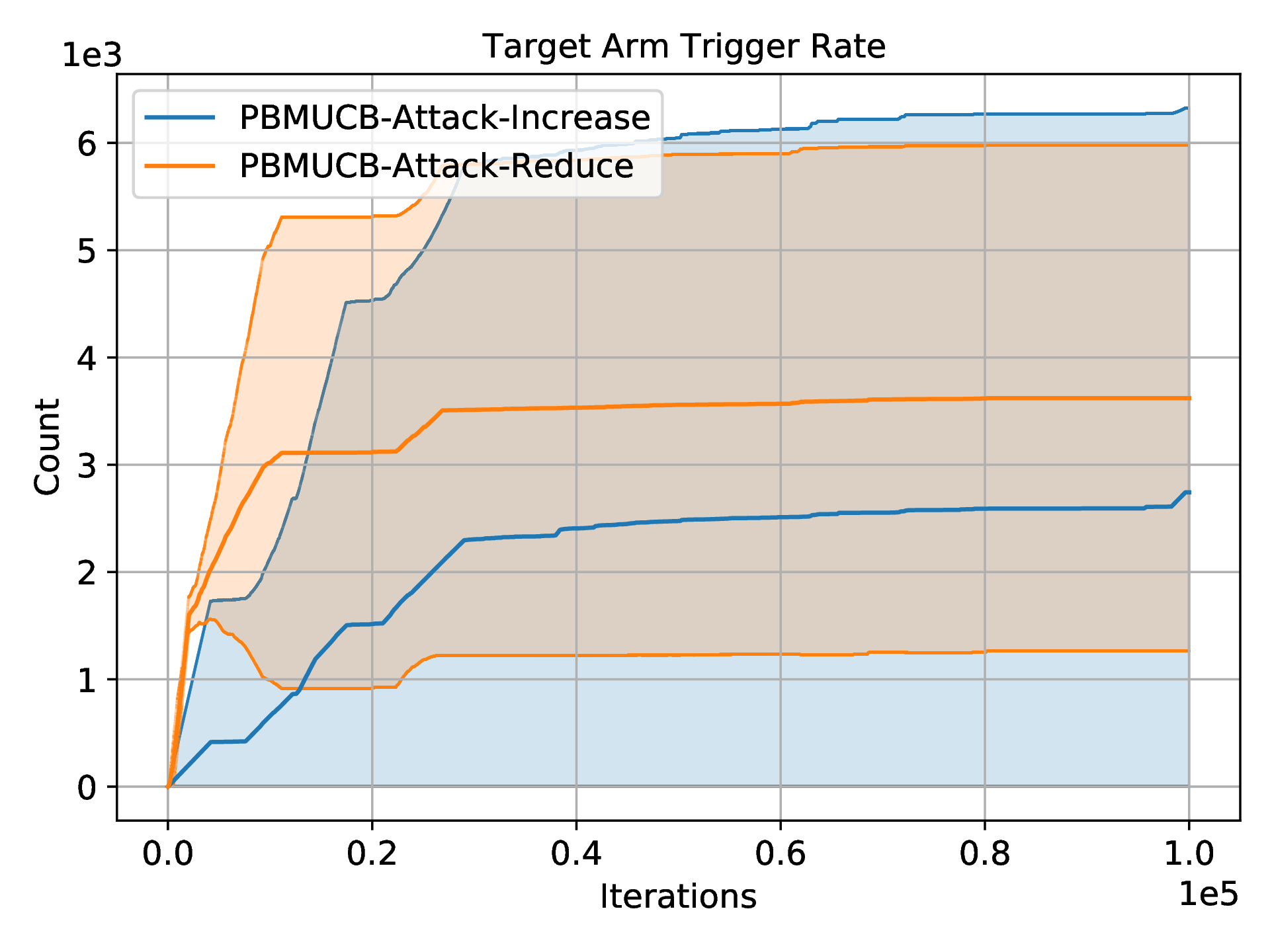}\label{fig:additional_baseline_pbm_rate}}
\caption{ Additional baseline: (a) the total cost spend in the cascade model, (b) $\bbN_t(\tilde{a})$ in the cascade model, (c) the total cost spend in the cascade model and (d) $\bbN_t(\tilde{a})$ in the position-based model. We report averaged result
and variance of 10 runs.}
\label{fig:additional_baseline}
\end{figure*}

\begin{figure*}
	\centering  
	\subfigbottomskip=0pt 
	\subfigcapskip=-5pt 
	\subfigure[Cost in CM.]{
\includegraphics[width=0.43\linewidth]{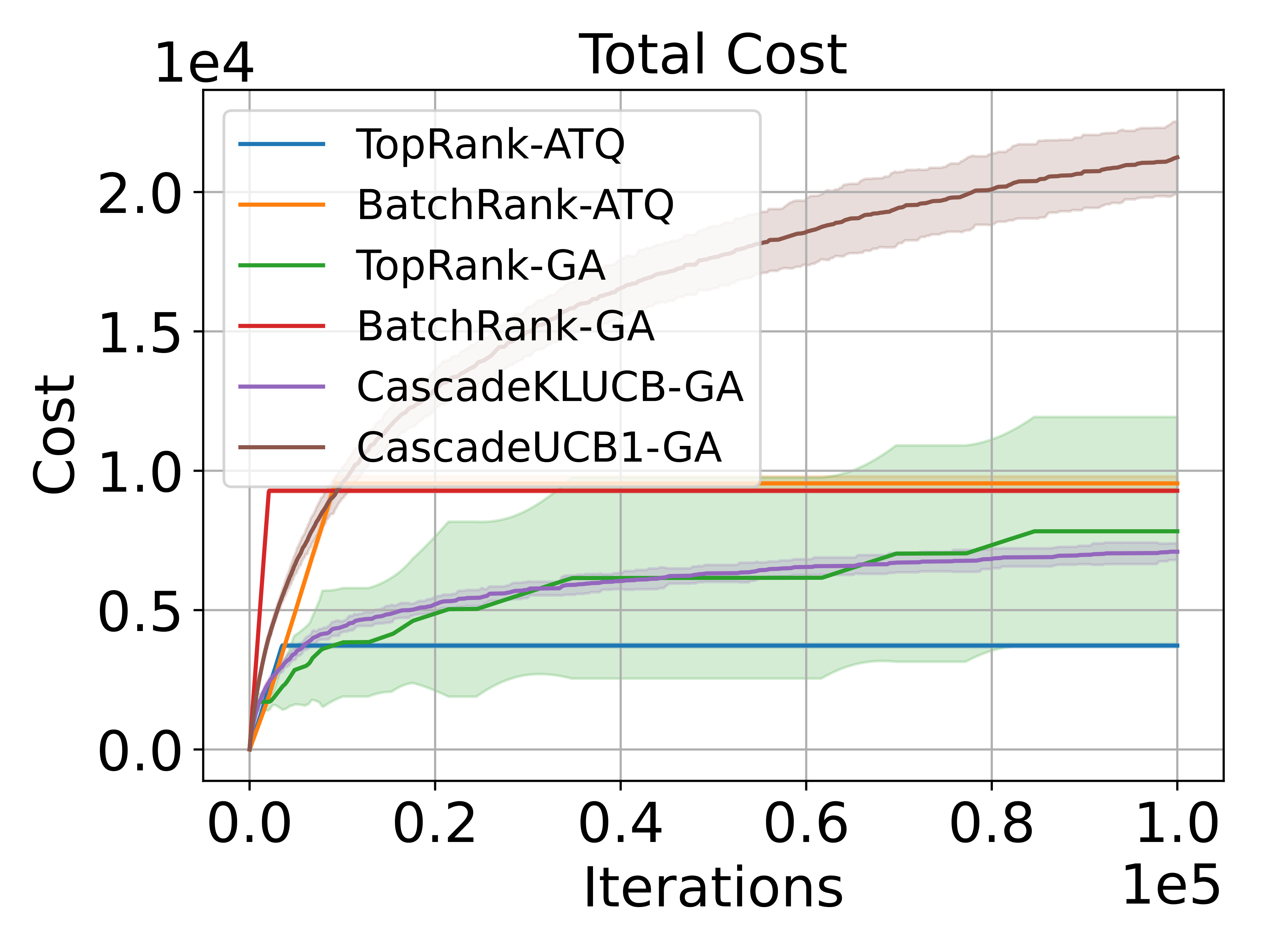}\label{fig:CM_movielens_cost}}\quad\quad\quad\quad
	\subfigure[$\bbN_t(\tilde{a})$ in CM.]{
\includegraphics[width=0.43\linewidth]{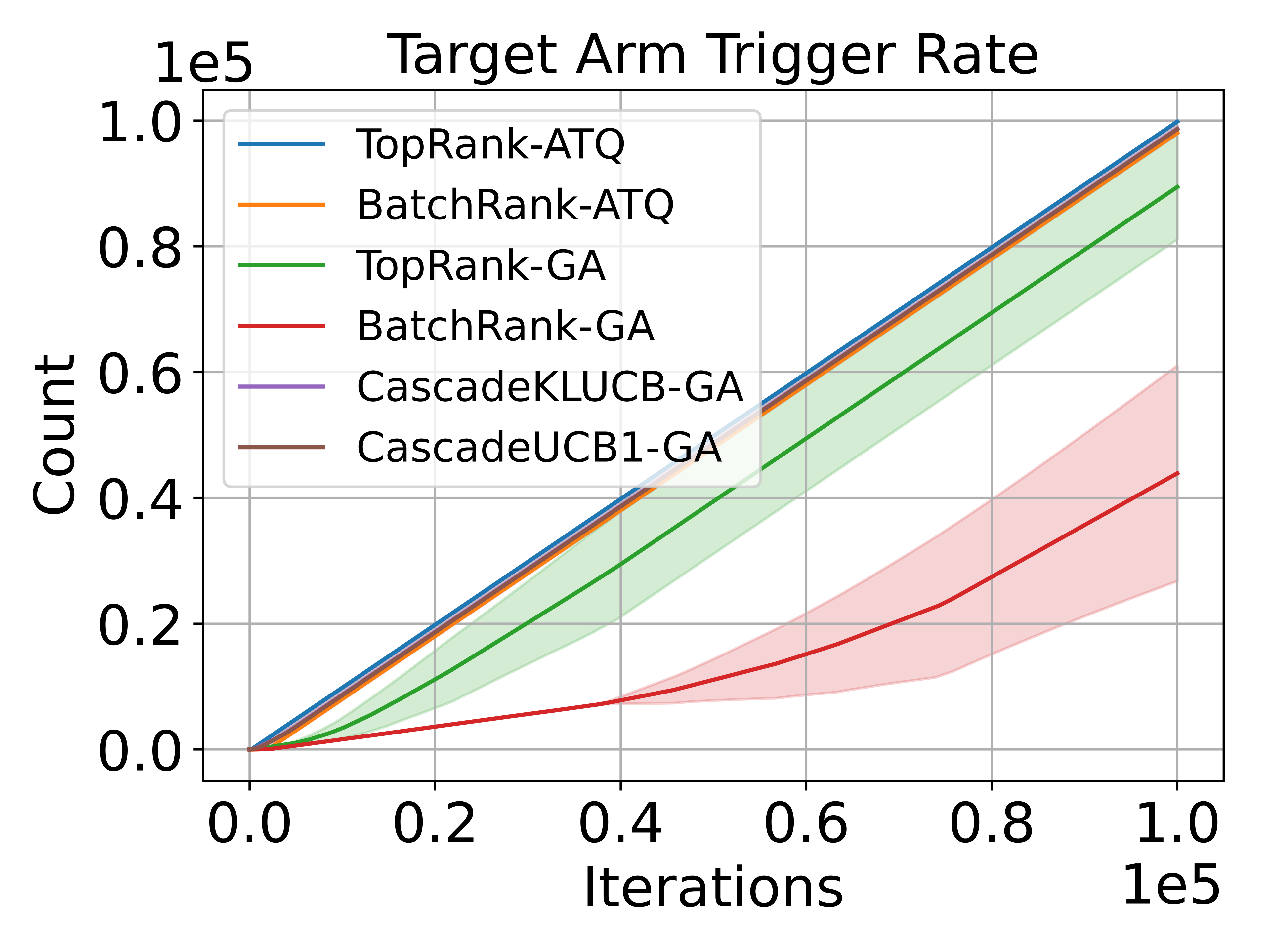}\label{fig:CM_movielens_rate}}\\
\subfigure[Cost in PBM.]{
\includegraphics[width=0.43\linewidth]{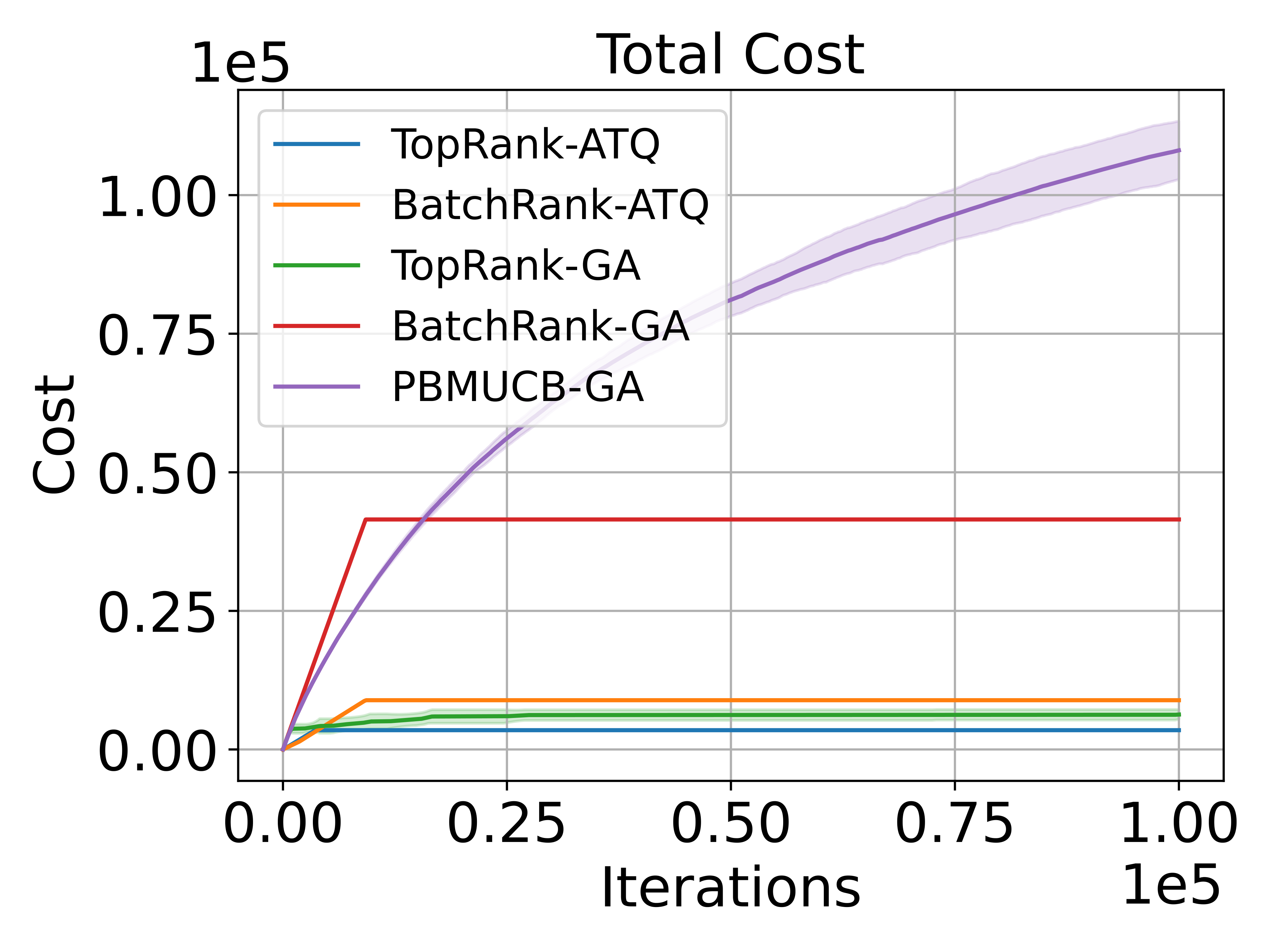}\label{fig:PBM_movielens_cost}}\quad\quad\quad\quad
\subfigure[$\bbN_t(\tilde{a})$ in PBM.]{
\includegraphics[width=0.43\linewidth]{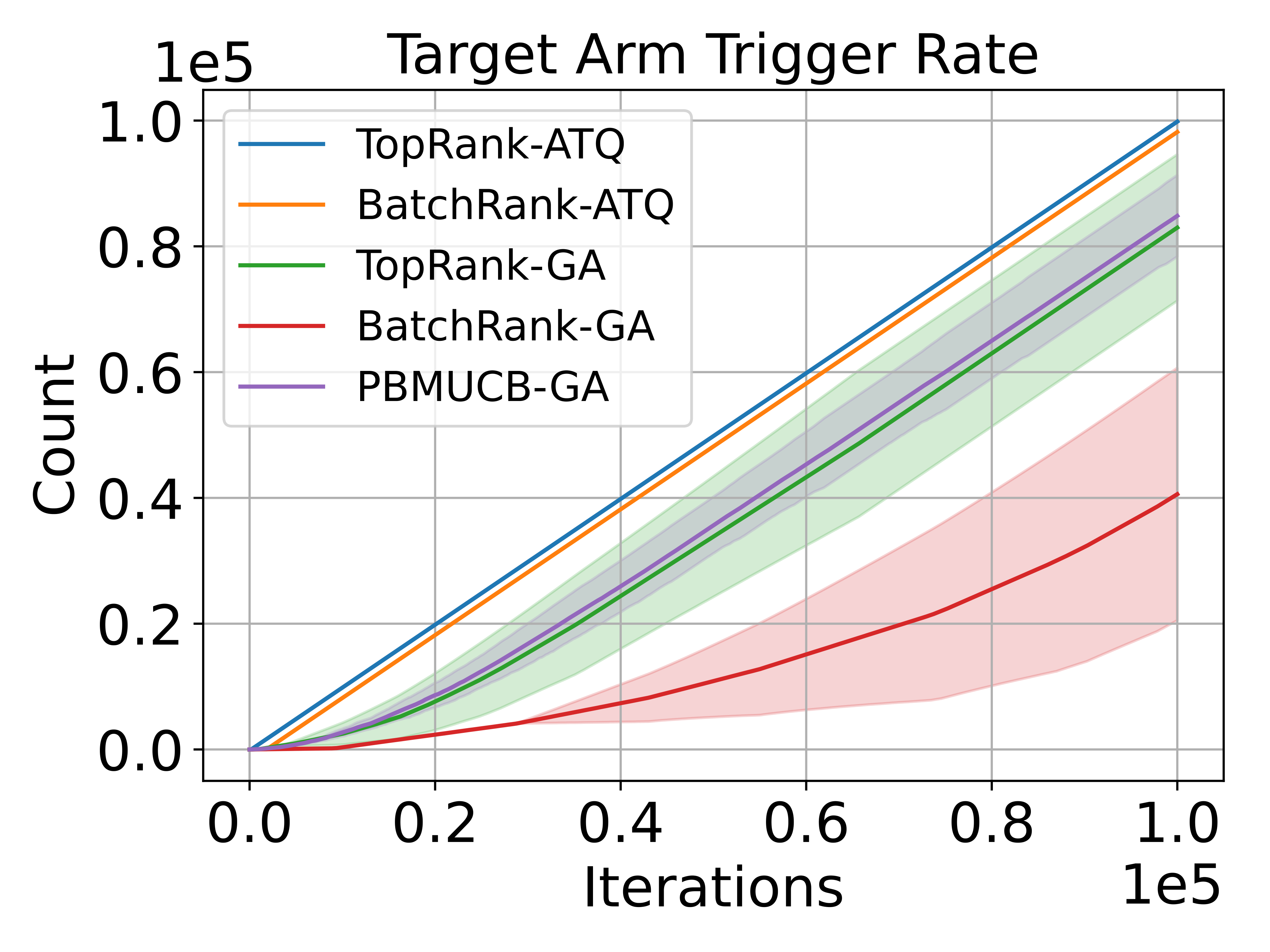}\label{fig:PBM_movielens_rate}}
\caption{ MovieLens experiment: (a) the total cost spend in the cascade model, (b) $\bbN_t(\tilde{a})$ in the cascade model, (c) the total cost spend in the cascade model and (d) $\bbN_t(\tilde{a})$ in the position-based model. We report averaged result
and variance of 10 runs.}
\label{fig:movielens_results}
\end{figure*}

\subsection{Additional  experiments on synthetic data}
Since we propose the first attack against OLTR, there are no existing baseline attack strategies in the literature to compare with. Nevertheless, We build two additional simple click-poisoning attack strategies as baselines. The setting of this experiment is the same as our  experiment on synthetic data in Section 5.
The first baseline directly reduces all the non-target item's click feedback to $0$ in the first two thousand rounds (indexed by Attack-reduce). The second baseline directly increases the click feedback of the target item to $1$ when the target item is in $\bR_t$ in the first two thousand rounds. These two attack strategies are tested to attack OLTR algorithms CascadeUCB1, KL-CascadeUCB, and PBM-UCB. The attack results are provided in Figure \ref{fig:additional_baseline}.

From Figure \ref{fig:additional_baseline}, we can observe that although baseline attacks achieve a sublinear cost (due to their stopping attack at around $2000$), none of them can fool at least one of the target algorithms to place the target item at the top position for $T-o(T)$ times. However, our attack strategies (including \texttt{GA} and \texttt{ATQ}) can efficiently attack all of these target OLTR algorithms, which are shown in Section 5. 

\subsection{Experiments on real-world data}
We also evaluate the proposed attacks on MovieLens dataset \citep{Harper2016TheMD}. We first split the dataset into train and test data subsets. Using the training data, we compute a $d$-rank SVD approximation, which is used to compute a mapping from movie rating to the probability that a user selected at random would rate the movie with 3 stars or above. We use the learned probability to simulate user's clicks given the ranking list.  We refer the reader to the Appendix C of \citep{Vial2022MinimaxRF} for further details. Figure \ref{fig:movielens_results} shows the attack results of our attack strategy averaged over $10$ rounds.

 We can observe that the trends in Figure \ref{fig:movielens_results} are similar to those in Figure \ref{fig:synthetic_results}, and the two attack algorithms are again able to efficiently fool the OLTR algorithms. In the cascade model, we see that successfully attacking CascadeKLUCB, TopRank, and BatchRank with \texttt{GA} only needs a relatively low cost, and the cost is higher when the target is CascadeUCB1. Besides, the \texttt{ATQ} strategy can still outperform the \texttt{GA} in $N_T(\tilde{a})$ when the target algorithms are TopRank and BatchRank. In the position-based model, the results are similar to the results in the cascade model, and the cost spent in the PBM-UCB is larger than the cost spent in the other algorithms.

\newpage

\section{Proof of Theorem \ref{proposition1}} 
 Recall the definition of the no-regret ranker, we can derive that the item with the highest attractiveness would be placed at the first position of $\bR_t$ for $T-o(T)$ times, otherwise, the regret would be linear. The reader should remember when the attacker implements Algorithm \ref{alg1}, the optimal list becomes $\bR^* = \bT = (\tilde{a},\eta_1,...,\eta_{K-1})$ due to the attractiveness of items belong to $\tilde{\bD}\backslash\bT$ is smaller than $\alpha(\eta_{K-1})$ (i.e., $\alpha(\eta_K) < \alpha(\eta_{K-1}) \le \alpha(\tilde{a})$). Based on Definition \ref{supdefinition} and Assumption \ref{supassumption}, the target item $\tilde{a}$ has the highest attractiveness and we can derive $\bE[\bbN_T(\tilde{a})] =  T-o(T)$.

Besides, according to the line 4 of the Algorithm \ref{alg1}, the cost of Algorithm \ref{alg1} can be bounded by
\begin{align}
     \bbC(T) = \bE\bigg[\sum_{t=1}^T\sum_{k=1}^K \bone \{ \ba_k^t \in \tilde{\bD} \backslash \bT \}\bigg] \le K\bE\bigg[\sum_{t=1}^T \bone \{ \bR_t \backslash \bT \not = \emptyset \}\bigg].
\end{align}
Due to the optimal list becomes $\bT$ during the attack, if $\bR_t\backslash\bT\not = \emptyset$, the per step regret is at least $\Delta_{\min} = \min_{\bR\in \Pi_K(\Lambda)}(\sum_{k = 1}^K v(\bT,\ba_k^t,k) - \sum_{k = 1}^K v(\gR,\ba_k^t,k)) > 0$, where $\Lambda$ consists of $\bT$ and $L+K-1$ items with attractiveness smaller equals then $\alpha(\eta_K)$. The cost can be bounded by
\begin{align}
    \bbC(T) \le K\bE\bigg[\sum_{t=1}^T \bone \{ \bR_t \backslash \bT \not = \emptyset \}\bigg] \le \frac{KR(T)}{\Delta_{\min}}.
\end{align}
Therefore, if the target ranker can achieve a sublinear regret $R(T) = o(T)$ in its click model under Assumption \ref{supassumption} (the definition of the no-regret ranker), the cost of Algorithm \ref{alg1} would be sublinear. According to Definition \ref{definition1} and our deduction, we can conclude that if a ranker belongs to no-regret rankers, it can be efficiently attacked by Algorithm \ref{alg1}. Here finish the proof of Theorem \ref{proposition1}.

\newpage

\section{Proof of Theorem \ref{suptheorem}}

\subsection{Introduction of CascadeUCB1}
The pseudo-code of the CascadeUCB1 is provided as follows.
\begin{algorithm}
\renewcommand{\algorithmicrequire}{\textbf{Input:}}
\renewcommand{\algorithmicensure}{\textbf{Output:}}
	\caption{The CascadeUCB1 \citep{Kveton2015CascadingBL}}
	\label{alg4}
	\begin{algorithmic}[1]
            \STATE \textbf{Input:} Item set $\bD$
            \STATE \textbf{for} $k=1:L$ \textbf{do}
             \STATE \quad
             Explore item $a_k$ and derive $\bc_k^0$
            \STATE \quad Set $\bN_0(a_k) = 1$ and $\hat{\alpha}_{1}(a_k) = \bc_k^0$
            \STATE \textbf{for} $t=1:T$ \textbf{do}
            \STATE \quad \textbf{for} $k=1:L$ \textbf{do}
            \STATE \quad \quad Compute $UCB_t(a_k)$
            \STATE \quad Let $\ba_1^t,...,\ba_K^t$ be $K$ items with largest UCBs and set $\bR_t = (\ba_1^t,...,\ba_K^t)$
            \STATE \quad Observe click feedback $\bC_t$
            \STATE \quad \textbf{for} $k=1:K$ \textbf{do}
            \STATE \quad \quad \textbf{if} $\ba^t_k$ is clicked \textbf{then}
            \STATE \quad\quad\quad Set $s = k$
            \STATE \quad \textbf{for} $k=1:L$ \textbf{do}
            \STATE \quad\quad Set $\bN_t(a_k) = \bN_{t-1}(a_k)$
            \STATE \quad \textbf{for} $k = 1:s$ \textbf{do} 
            \STATE \quad \quad Set $\bN_t(\ba_k^t) = \bN_{t}(\ba_k^t) + 1$
            \STATE \quad \quad $\hat{\alpha}_{\bN_t(\ba_k^t)}(\ba_k^t) = \frac{\bN_{t-1}(\ba_k^{t})\hat{\alpha}_{\bN_{t-1}(\ba_k^{t})}(\ba_k^{t}) + \bone\{s=k\}}{\bN_t(\ba_k^t)}$
	\end{algorithmic}  
\end{algorithm}

We let $\bN_t(a_k)$ denotes the number of item $a_k$ be examined till round $t$. The upper confidence bound  is defined as $UCB_t(a_k) = \hat{\alpha}_{\bN_{t-1}(a_k)}(a_k) + 3\sqrt{(\log(t-1))/\bN_{t-1}(a_k)}$.

\subsection{Proof of Theorem \ref{suptheorem}}

The proof of Theorem \ref{suptheorem} relies on the following lemmas.

\begin{lemma} [The Hoeffding inequality] Let $X_1, X_2,...,X_n$ i.i.d drawn from a Bernoulli distribution, $\bar{X} = \frac{1}{n}\sum_{i=1}^n X_i$ and $\bE[X]$ be the mean, then 
\begin{align}
    P(\bar{X} - \bE[X] \le -a) \le e^{-na^2/2}.
\end{align}
\end{lemma}

\begin{lemma} \label{suplemma} Consider item $a_1$ is the item with the highest attractiveness and $a_k \neq a_1$. When the principal runs the CascadeUCB1, the expected number of $a_k$ be placed at the first position till round $T$ can be bounded by $\bE[\bbN_{T}(a_k)] \le 3 + 81\log(T)/\Delta_{k}^2$, where $\Delta_k = \alpha(a_1) - \alpha(a_k)$.
\end{lemma}

\begin{proof}[Proof of Lemma \ref{suplemma}] 
 We first decompose $\bE[\bbN_{T}(a_k)]$ as follows
\begin{align}
    \begin{split}
    \bE[\bbN_{T}(a_k)] &\le 1 + \bE\bigg[\sum_{t = 1}^T \bone\bigg\{\ba_1^t = a_k,\ \bbN_{t-1}(a_k) < \frac{81\log(T)}{\Delta_{k}^2}\bigg\} \bigg] + \bE\bigg[\sum_{t = 1}^T \bone\bigg\{\ba_1^t = a_k,\ \bbN_{t-1}(a_k) \ge \frac{81\log(T)}{\Delta_{k}^2}\bigg\} \bigg]\\
    &\le 1 + \frac{81\log(T)}{\Delta_{k}^2} + \bE\bigg[\sum_{t = 1}^T \bone\bigg\{\ba_1^t = a_k,\ \bbN_{t-1}(a_k) \ge \frac{81\log(T)}{\Delta_{k}^2}\bigg\} \bigg]\\
    & \le 1 + \frac{81\log(T)}{\Delta_{k}^2} + \sum_{t = 1}^T P\bigg(UCB_t(a_k) \ge UCB_t(a_1),\ \bbN_{t-1}(a_k) \ge \frac{81\log(T)}{\Delta_{k}^2}\bigg).
    \end{split}
\end{align}
By union bound, we then decompose and bound probability $P\Big(UCB_t(a_k) \ge UCB_t(a_1),\ \bbN_{t-1}(a_k) \ge 81\log(T)/\Delta_{k}^2 \Big)$
\begin{align}
    \begin{split}
        &P\bigg(UCB_t(a_k)  \ge UCB_t(a_1),\   \bbN_{t-1}(a_k) \ge \frac{81\log(T)}{\Delta_{k}^2}\bigg)\\ 
        \le &\sum_{\lambda = 1}^{t-1}\sum_{\sigma \ge \frac{81\log(T)}{\Delta_{k}^2}}^{t-1} P\bigg(UCB_t(a_k)  \ge UCB_t(a_1)  \bigg\vert \bN_{t-1}(a_k) = \sigma,\ \bN_{t-1}(a_1) = \lambda \bigg).
    \end{split}
\end{align}
The inequality holds due to $\bN_{t-1}(a_k) \ge \bbN_{t-1}(a_k)$. 
We further upper bound $P\Big(UCB_t(a_k)  \ge UCB_t(a_1)  \Big\vert \bN_{t-1}(a_k) = \sigma,\ \bN_{t-1}(a_1) = \lambda \Big)$. Consider for $1 \le \lambda \le t-1$ and $81\log(T)/\Delta_{k}^2 \le \sigma \le t-1$, we have
\begin{align}\label{5}
    \begin{split}
        &P\bigg(UCB_t(a_k)  \ge UCB_t(a_1)  \bigg\vert \bN_{t-1}(a_k) = \sigma,\ \bN_{t-1}(a_1) = \lambda \bigg)\\
        \le& P\bigg(\hat{\alpha}_{\bN_{t-1}(a_k)}(a_k) + 3\sqrt{\frac{\log(T)}{\bN_{t-1}(a_k)}} + \frac{\Delta_{k}}{3} \ge  \hat{\alpha}_{\bN_{t-1}(a_1)}(a_1) + 3\sqrt{\frac{\log(T)}{\bN_{t-1}(a_1)}} \bigg\vert \bN_{t-1}(a_k) = \sigma,\ \bN_{t-1}(a_1) = \lambda \bigg)\\
         \le& P\bigg(\hat{\alpha}_{\bN_{t-1}(a_k)}(a_k) + \frac{2\Delta_{k}}{3} \ge  \hat{\alpha}_{\bN_{t-1}(a_1)}(a_1) + 3\sqrt{\frac{\log(T)}{\bN_{t-1}(a_1)}} \bigg\vert \bN_{t-1}(a_k) = \sigma,\  \bN_{t-1}(a_1) = \lambda \bigg)\\
         \le &P\bigg(\hat{\alpha}_{\bN_{t-1}(a_k)}(a_k) + \alpha(a_1) - \alpha(a_k) \ge \frac{\Delta_{k}}{3} + \hat{\alpha}_{\bN_{t-1}(a_1)}(a_1) + 3\sqrt{\frac{\log(T)}{\bN_{t-1}(a_1)}} \bigg\vert \bN_{t-1}(a_k) = \sigma,\  \bN_{t-1}(a_1) = \lambda \bigg).
    \end{split}
\end{align}
The first inequality relies on the definition of the $UCB_t(a_k)$. The second inequality holds because $\sigma \ge 81\log(T)/\Delta_{k}^2$. The third inequality holds because $\Delta_{k} = \alpha(a_1) - \alpha(a_k)$. 

Based on the Hoeffding inequality, we have for any $\lambda \ge 1$ and $\sigma \ge 81\log(T)/\Delta_{k}^2$ 
\begin{align}
    \begin{split}\label{2}
        &P\bigg(\alpha(a_1) - \hat{\alpha}_{\bN_{t-1}(a_1)}(a_1) \ge 3\sqrt{\frac{\log(T)}{\bN_{t-1}(a_1)}} \bigg\vert \bN_{t-1}(a_1) = \lambda \bigg) \le \frac{1}{T^{9/2}} \\
        &P\bigg(\hat{\alpha}_{\bN_{t-1}(a_k)}(a_k) - \alpha(a_k) \ge \frac{\Delta_k}{3} \bigg\vert \bN_{t-1}(a_k) = \sigma \bigg) \le \frac{1}{T^{9/2}}.
    \end{split}
\end{align}
 The last term of (\ref{5}) can be further bounded by
\begin{align}\label{6}
    \begin{split}
        &P\bigg(\hat{\alpha}_{\bN_{t-1}(a_k)}(a_k) + \alpha(a_1) - \alpha(a_k) \ge \frac{\Delta_{k}}{3} + \hat{\alpha}_{\bN_{t-1}(a_1)}(a_1) + 3\sqrt{\frac{\log(T)}{\bN_{t-1}(a_1)}} \bigg\vert \bN_{t-1}(a_k) = \sigma,\  \bN_{t-1}(a_1) = \lambda \bigg)\\
        \le & P\bigg(\hat{\alpha}_{\bN_{t-1}(a_k)}(a_k) -  \alpha(a_k) \ge \frac{\Delta_{k}}{3} \bigg\vert \bN_{t-1}(a_k) = \sigma\bigg) + P\bigg( \alpha(a_1) -  \hat{\alpha}_{\bN_{t-1}(a_1)}(a_1) \ge 3\sqrt{\frac{\log(T)}{\bN_{t-1}(a_1)}} \bigg\vert \bN_{t-1}(a_1) = \lambda \bigg)\\
        \le& \frac{1}{T^{9/2}} + \frac{1}{T^{9/2}} \le \frac{2}{T^{9/2}}.
    \end{split}
\end{align}
The first inequality holds due to the union bound and the last inequality holds due to (\ref{2}).

With the fact that
\begin{align}\label{15}
    \sum_{t=1}^T\sum_{\lambda = 1}^{t-1}\sum_{\sigma \ge \frac{81\log(T)}{\Delta_{k}^2}}^{t-1}\frac{2}{T^{9/2}} \le \frac{2}{T^{3/2}} \le 2.
\end{align}
In the light of (\ref{15}), the total expected number of $a_k$ been placed at the first position can be bounded by
\begin{align}
    \bE[\bbN_T(a_k)] \le 3 + \frac{81\log(T)}{\Delta_{k}^2}.
\end{align}
Here finish the proof of Lemma \ref{suplemma}.
\end{proof}

\begin{proof}[Proof of Theorem \ref{suptheorem}] 
With Lemma \ref{suplemma}, we can bound the total expected number of $a_k \not = a_1$ being placed at the first position till round $T$. Thus, from round $1$ to round $T$, the expected number of CascadeUCB1 place item $a_1$ at the first position satisfies
\begin{align}\label{16}
    \bbN_T(a_1) \ge T - \sum_{k=2}^L \bigg( 3 + \frac{81\log(T)}{\Delta_{k}^2}\bigg).
\end{align}
Remember when the attacker implements attack Algorithm \ref{alg1}, the target item would become the item with the highest attractiveness. The rest of the items consist of $\{\eta_k\}_{k=1}^{K-1}$ and $L+K-1$ items with attractiveness at most $\alpha(\eta_K)$. Therefore, when Algorithm \ref{alg1} attacks the CascadeUCB1, $\bbN_T(\tilde{a})$ can be lower bounded by
\begin{align}
    \bE[\bbN_T(\tilde{a})] \ge T - \sum_{k=1}^{K-1}\bigg( \frac{3 + 81\log(T)}{(\alpha(\tilde{a}) - \alpha(\eta_k))^2} \bigg) - \sum_{s=1}^{L+K-1}\bigg( \frac{3 + 81\log(T)}{(\alpha(\tilde{a}) - \alpha(\eta_K))^2}\bigg).
\end{align}
Besides, according to the line 4 of Algorithm \ref{alg1}, the cost of Algorithm \ref{alg1} attack CascadeUCB1 can be bounded by
\begin{align}
    \bbC(T) \le K\bE\bigg[\sum_{t=1}^T \bone \{ \bR_t \backslash \bT \not = \emptyset \}\bigg].
\end{align}
It is worth noting that Algorithm \ref{alg1} only manipulates items in list $\bR_t$, hence the cost generates in one round is at most $K$. Recall the definition of regret in the cascade model
\begin{align}
\begin{split}
    R(T) & = \bE \bigg[ T\sum_{k=1}^K v(\bT,\ba_{k}^t,k) - \sum_{t=1}^T\sum_{k=1}^K v(\bR_t,\ba_{k}^t,k) \bigg]\\
    &= \bE \bigg[ T \bigg(1 - (1-\alpha(\tilde{a}))\prod_{k=1}^{K-1} (1 - \alpha(\bar{a}_k))\bigg) - \sum_{t = 1}^T \bigg(1 - \prod_{k=1}^{K} (1 - \alpha(\ba_k^t))\bigg) \bigg].
\end{split}
\end{align}
The total regret is generated by $K$ positions. Algorithm \ref{alg1} only attacks when $\bR_t\backslash\bT\not=\emptyset$. And situation $\bR_t\backslash\bT\not=\emptyset$ implies there is at least one item $\not\in\bT$ be placed in the $\bR_t$ and its attractiveness is reduced to at most $\alpha(\eta_K)$. Due to when $\bR_t\backslash\bT \not= \emptyset$, the number of items is placed in $\bR_t$ and belongs to $\tilde{\bD}\backslash\bT$ is at least $1$. Then for the cascade model, the regret generates in round $t$ is at least 
\begin{align}
\begin{split}
   &\sum_{k=1}^K \bigg(v(\bT,\ba_{k}^t,k) - v(\bR_t,\ba_{k}^t,k)\bigg)\\
    \ge& 1 - (1-\alpha(\tilde{a})) \prod_{k=1}^{K-1} (1 - \alpha(\eta_{k})) - 1 + (1-\alpha(\tilde{a}))(1-\alpha(\eta_K)) \prod_{k=1}^{K-2} (1 - \alpha(\eta_{k}))\\  =&
 (\alpha(\eta_1) - \alpha(\eta_K)) (1-\alpha(\tilde{a})) \prod_{k=1}^{K-2} (1 - \alpha(\eta_{k})).
\end{split}
\end{align}
The first inequality holds due to $\alpha(\eta_{K-1})$ has the lowest attractiveness in $\bT$. With the above derivation, we can derive when $\bR_t \backslash \bT \not = \emptyset$, the regret generates in each round is at least $(\alpha(\eta_{1}) - \alpha(\eta_K)) (1-\alpha(\tilde{a})) \prod_{k=1}^{K-2} (1 - \alpha(\eta_{k}))$. With this in mind, we can further bound the total cost by
\begin{align}
\bbC(T) \le K\bE\bigg[\sum_{t=1}^T \bone \{ \bR_t \backslash \bT \not = \emptyset \}\bigg] \le \frac{KR(T)}{(\alpha(\eta_{1}) - \alpha(\eta_K)) (1-\alpha(\tilde{a})) \prod_{k=1}^{K-2} (1 - \alpha(\eta_{k}))}
\end{align}
Due to the regret of the CascadeUCB1 satisfies $R(T) = o(T)$, the cost of Algorithm \ref{alg1} would be sublinear. We conclude that the CascadeUCB1 can be efficiently attacked by Algorithm \ref{alg1}. Here finish the proof of Theorem \ref{suptheorem}.
\end{proof}
\newpage

\section{Proof of Theorem \ref{theorem1}}

\subsection{Introduction of BatchRank}

We here specifically illustrate details of
BatchRank. The pseudo-code of the BatchRank is provided as follows.

\begin{algorithm}
\renewcommand{\algorithmicrequire}{\textbf{Input:}}
\renewcommand{\algorithmicensure}{\textbf{Output:}}
	\caption{BatchRank \citep{zoghi2017onlineLT}}
	\label{alg4}
	\begin{algorithmic}[1]
            \STATE \textbf{Initialize:} 
             $b_{\max} = 1$, $\bI_1 = (\bI_1(1) = 1, \bI_1(2) = K)$, $\ell_1 = 0$, $B_{1,0} = \bD$, $\bB = \{1\}$
             \STATE \textbf{for} $b = 1:K$ \textbf{do}
             \STATE \quad \textbf{for} $\ell = 0:T-1$ \textbf{do}
             \STATE \quad \quad \textbf{for all} $a_k \in \bD$ \textbf{do}
               \STATE \quad \quad \quad $\bc_{b,\ell}(a_k) = 0$, $\bn_{b,\ell}(a_k) = 0$
            \STATE \textbf{for} $t = 1:T$ \textbf{do}
            \STATE \quad \textbf{for all} $b\in \bB$ \textbf{do} 
            \STATE \quad \quad DisplayBatch($t$,$b$)
            \STATE \quad \textbf{for all} $b\in \bB$ \textbf{do} 
            \STATE \quad \quad CollectClicks($t$,$b$)
            \STATE \quad \textbf{for all} $b\in \bB$ \textbf{do} 
            \STATE \quad \quad UpdateBatch($t$,$b$)
	\end{algorithmic}  
\end{algorithm}

\begin{algorithm}
\renewcommand{\algorithmicrequire}{\textbf{Input:}}
\renewcommand{\algorithmicensure}{\textbf{Output:}}
	\caption{DisplayBatch}
	\label{alg4}
	\begin{algorithmic}[1]
            \STATE \textbf{Input:} batch index $b$, time $t$
              \STATE  Set $\ell = \ell_b$
            \STATE  Let $a_1,...,a_{\vert 
            B_{b,\ell} \vert}$ be a random permutation of items in  $B_{b,\ell}$ such that $\bn_{b,\ell}(a_1) \le...\le \bn_{b,\ell}(a_{\vert B_{b,\ell} \vert})$
             \STATE  Let $\pi \in \prod_{len(b)}([len(b)])$ be a random permutation of position assignments
             \STATE \textbf{for} $k = \bI_{b}(1):\bI_{b}(2)$ \textbf{do}
             \STATE \quad  $\ba_{k}^t = a_{\pi (k - \bI_b(1) + 1)}$
	\end{algorithmic}  
\end{algorithm}

\begin{algorithm}
\renewcommand{\algorithmicrequire}{\textbf{Input:}}
\renewcommand{\algorithmicensure}{\textbf{Output:}}
	\caption{CollectClicks}
	\label{alg4}
	\begin{algorithmic}[1]
            \STATE \textbf{Input:} batch index $b$, time $t$
            \STATE  Set $\ell = \ell_b$ and $\bn_{\min} = \min_{a_k\in B_{b,\ell}} \bn_{b,\ell}(a_{k})$
            \STATE Receive the click feedback $\bC_t = (\bc_1^t,...,\bc_L^t)$
             \STATE \textbf{for} $k = \bI_b(1):\bI_b(2)$ \textbf{do}
             \STATE \quad \textbf{if} $\bn_{b,\ell}(\ba_{k}^t) = \bn_{\min}$ \textbf{then}
             \STATE \quad \quad Set $\bc_{b,\ell}(\ba_{k}^t) = \bc_{b,\ell}(\ba_{k}^t) + \sum_{s=1}^L \bc_{s}^t\bone\{a_s = \ba_k^t\}$ and $\bn_{b,\ell}(\ba_{k}^t) = \bn_{b,\ell}(\ba_{k}^t) + 1$
	\end{algorithmic}  
\end{algorithm}

\begin{algorithm}
\renewcommand{\algorithmicrequire}{\textbf{Input:}}
\renewcommand{\algorithmicensure}{\textbf{Output:}}
	\caption{UpdateBatch}
	\label{alg4}
	\begin{algorithmic}[1]
            \STATE \textbf{Input:} batch index $b$, time $t$
            \STATE  Set $\ell = \ell_b$
            \STATE  \textbf{if}  $ \min_{a_k \in B_{b,\ell}} \bn_{b,\ell}(a_k)  =\bn_{\ell} $
             \STATE  \quad \textbf{for all} $a_k \in B_{b,\ell}$ \textbf{do}
             \STATE \quad \quad Compute $U_{b,\ell}(a_k)$ and $L_{b,\ell}(a_k)$
             \STATE  \quad Let $a_1,...,a_{\vert 
            B_{b,\ell} \vert}$ be any permutation of  $B_{b,\ell}$ such that $L_{b,\ell}(a_1) \ge...\ge L_{b,\ell}(a_{\vert B_{b,\ell} \vert})$
            \STATE  \quad \textbf{for} $k = 1:len(b)$ \textbf{do}
            \STATE  \quad \quad Set $B_{k}^+ = \{a_1,...,a_k\}$ and $B_k^- = B_{b,\ell}\backslash B^+_k$
            \STATE  \quad \textbf{for} $k = 1:len(b) - 1$ \textbf{do}
            \STATE  \quad \quad \textbf{if} $L_{b,\ell}(a_k) > \max_{a_k \in B_{k}^-}U_{b,\ell}(a_k)$ \textbf{then}
             \STATE  \quad \quad \quad Set $s = k$ 
             \STATE  \quad \textbf{if} $s = 0$ and $\vert B_{b,\ell} \vert > len(b)$ \textbf{then} 
             \STATE  \quad \quad
             Set $B_{b,\ell + 1} = \{a_k \in B_{b,\ell}:U_{b,\ell}(a_k) \ge L_{b,\ell}(a_{len(b)})\}$ and $\ell = \ell + 1$
             \STATE  \quad \textbf{else if} $s > 0$ \textbf{then}
             \STATE  \quad \quad Set $\bB = \bB \bigcup \{b_{\max} + 1,b_{\max} + 2\}\backslash\{b\}$, $B_{b_{\max} + 1,0} = B_{s}^+$,  $B_{b_{\max} + 2,0} = B_{s}^-$, $\ell_{ b_{\max} + 1} = 0$
             \STATE \quad \quad $\ell_{b_{\max} + 2} = 0$,
             $\bI_{b_{\max} + 1} = (\bI_{b}(1),\bI_{b}(1) + s -1)$, $\bI_{b_{\max} + 2} = (\bI_{b}(1) + s,\bI_{b}(2))$, 
             $b_{\max} = b_{\max} + 2$
	\end{algorithmic}  
\end{algorithm}

The BatchRank explores items with batches, which are indexed by $b$. 
 The BatchRank would begin with stage $\ell_1 = 0$, batch index $b =1$, and the first batch $B_{b,\ell_1} = \bD$. The first position in batch $b$ is indexed by $\bI_b(1)$ and the last position is indexed by $\bI_b(2)$, and the number of positions in batch $b$ is $\text{len}(b) = \bI_b(1) - \bI_b(2) + 1$. The first batch $B_{b,\ell_1}$ contains all the positions in $\bR_t$. In stage $\ell_1$, every item in $B_{b,\ell_1}$ would be explored for
 $\bn_{\ell_1} = 16\tilde{\Delta}_{\ell_1} ^{-2}\log(T)$ times (DisplayBatch) and $\tilde{\Delta}^{-1}_{\ell_1} = 2^{-\ell_{1}}$. Afterward, the BatchRank would  estimate the attractiveness of item $a_k$ as 
 \begin{align}\label{6}
    \hat{\bc}_{b,\ell}(a_k) = \bc_{b,\ell}(a_k)/
    \bn_{\ell}.
\end{align}
  After the CollectClicks section, the ranker would compute the KL-upper confidence bound and lower confidence bound \citep{Garivier2011TheKA,zoghi2017onlineLT} for every item in the batch, denote as $U_{b,\ell}(a_k)$ and $L_{b,\ell}(a_k)$
 \begin{align}
     \begin{split}
         & U_{b,\ell}(a_k) = \argmax_{q \in [\hat{\bc}_{b,\ell}(a_k),1]}\{\bn_\ell D_{KL}(\hat{\bc}_{b,\ell}(a_k) \Vert q) \le \log(T) + 3\log\log(T)\}\\
         &L_{b,\ell}(a_k) = \argmin_{q \in [0, \hat{\bc}_{b,\ell}(a_k)]}\{\bn_\ell D_{KL}(\hat{\bc}_{b,\ell}(a_k) \Vert q) \le \log(T) + 3\log\log(T)\}\\
     \end{split}
 \end{align}
 where $D_{KL}$ represents the \emph{Kullback-Leibler divergence} between Bernoulli random variables with means $p$ and $q$.
 In the UpdateBatch section, all the items in batch $B_{b,\ell_1}$ would be placed by order $a_1,..., a_{\vert B_{b,\ell_1} \vert}$, where $L_{b,\ell_1}(a_1)\ge,...,\ge L_{b,\ell_1}(a_{\vert B_{b,\ell_1} \vert})$. The BatchRank would compare the first $len(b)-1$ item's lower confidence bound to the maximal upper confidence bound in $B_k^-$. If $L_{b,\ell_1}(a_k) > \max_{a_k \in B_{k}^-}U_{b,\ell_1}(a_k)$, the BatchRank would set $s = k$. Ones $s>0$, the batch would spilt from position $s$ and the ranker derives sub-batches $B_{b+1,\ell_{2}}$ and $B_{b+2,\ell_{3}}$. Sub-batch $B_{b+1,\ell_{2}}$ contains $s$ items and the first $s$ positions in $\bR_t$ and sub-batch $B_{b+2,\ell_{3}}$ contains $L-s$ items and positions from $s$ to $K$. The BatchRank would restart with stages $\ell_{2} = 0$ and $\ell_{3} = 0$ and sub-batches $B_{b+1,\ell_{2}}$ and $B_{b+2,\ell_{3}}$.  The batches would recursively run and split until round $T$. 

  \subsection{Missing example in section 4.2}

\paragraph{Example 2.}\label{exp1}
Consider the item set $\bD = \{a_1,a_2,...,a_{10}\}$ contains $10$ items. The target item is $\tilde{a} = a_1$ with arbitrary attractiveness and item list $\bR_t$ contains $3$ positions. The BatchRank begins with $B_{1,\ell_1} = \bD$ and recursively explores items. The attacker begins the attack at round $1$ and ends at $T_1$. During $T_1$ rounds, the attacker returns $\tilde{\bc}_{k}^t = 1$ if $\tilde{a} = a_{k}$ and $\tilde{a} \in \bR_t$, otherwise $\tilde{\bc}_{k}^t = 0$. This fraudulent strategy misleads BatchRank to believe the attractiveness of the non-target items $a_k \neq \tilde{a}$ is $0$ and the attractiveness of $\tilde{a}$ is $1$. After BatchRank finishes the split section, the batch $B_{1,\ell_1}$ split into $B_{2,\ell_2} = \{a_1\}$ and $B_{3,\ell_3} = \{a_2,...,a_{10}\}$. Sub-batch $B_{2,\ell_2}$ contains position $1$ (i.e., $\ba_1^t$) and $B_{3,\ell_3}$ contains position $2$ and $3$ (i.e., $\ba_2^t$ and $\ba_3^t$). The process is illustrated in Figure \ref{fig3}.
The target item $\tilde{a}$ would always be placed at the first position of the $\bR_t$ in the next $T - T_1$ rounds due to $B_{2,\ell_2}$ only contains $\tilde{a}$. The attack cost would be $o(T)$ because the attacker only manipulates the click feedback for $o(T)$ rounds.

\begin{figure}
\centering  
\includegraphics[width=0.7\linewidth]{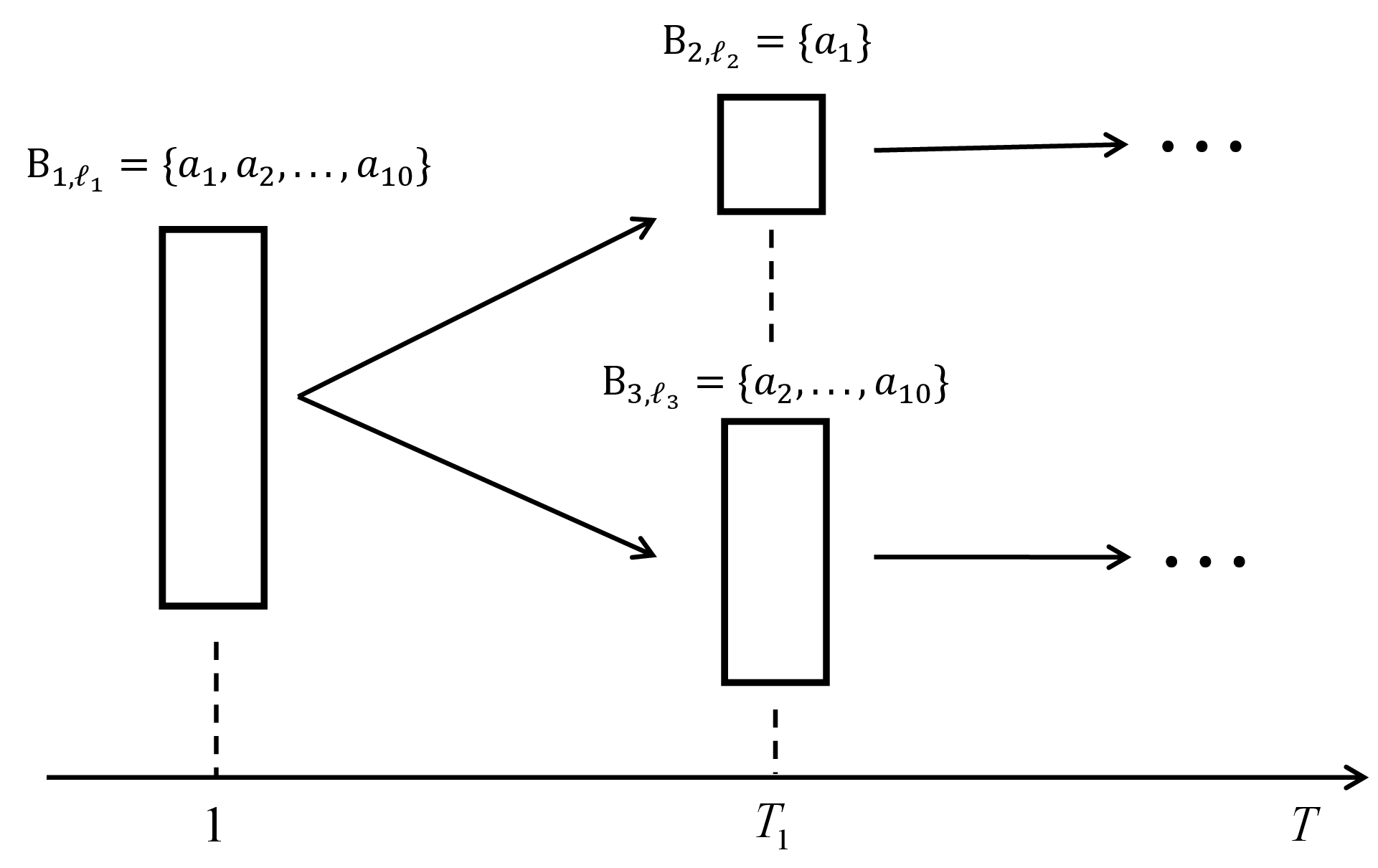}
\caption{Process of the Algorithm \ref{alg2} attacks BatchRank.} 
\label{fig3}
\end{figure}

\subsection{Proof of Theorem \ref{theorem1}}
The proof of Theorem \ref{theorem1} relies on the following Lemma \ref{lemma1}.

 \begin{lemma}\label{lemma1} 
 The attacker utilizes Algorithm \ref{alg2} to manipulate the returned click feedback of the BatchRank. After $16L\log(T)$ rounds attack and the BatchRank begins its first split. The upper confidence bound and lower confidence bound of every non-target item satisfies $L_{b,\ell_1}(a_k) = 0$ and $U_{b,\ell_1}(a_k) = 1 - (T\log(T)^3)^{ -1/\bn_{\ell_1} }$. The lower confidence bound and the upper confidence bound of the target item are $L_{b,\ell_1}(\tilde{a}) = 1$ and $U_{b,\ell_1}(\tilde{a}) = 1$.
 \end{lemma}

\begin{proof}[Proof of Lemma \ref{lemma1}] According to the introduction of BatchRank, the estimated click probability of an arbitrary item is written as (\ref{6}) and $\bc_{b,\ell_1}(a_k)$ is at most $16\log(T)$ in the first stage ($\ell_1 = 0$ and $\tilde{\Delta}_{\ell_1}^{-2} = 2^{2\ell_1} = 1$). Recall our attack Algorithm \ref{alg2} returns $\tilde{\bc}_{k}^t = 0$ when $a_k \neq \tilde{a}$ and $a_k\in\bR_t$. Thus, the total collected click number of the non-target item is $\bc_{b,\ell_1}(a_k) = 0$, and the estimated click probability is $\hat{\bc}_{b,\ell_1}(a_k) = 0$. 

We first introduce the definition of the KL-divergence
\begin{align}
\begin{split}
    D_{KL}(p \Vert q) = p\log(\frac{p}{q}) + (1 - p)\log(\frac{1 - p}{1 - q}).
\end{split}
\end{align}
By convenience, we define $0\log(0) = 0\log(0/0) = 0$ and $x\log(x/0) = + \infty$ for $x > 0$ \citep{Garivier2011TheKA}. With this knowledge, we can derive the upper confidence bound of the non-target item in stage $\ell_1$ 
\begin{align}
\begin{split}
U_{b,\ell_1}(a_k) &= \argmax_{q \in [\hat{\bc}_{b,\ell_1}(a_k),1]}\{\bn_{\ell_1} D_{KL}(\hat{\bc}_{b,\ell_1}(a_k) \Vert q) \le \log(T) + 3\log\log(T)\}\\
& = \argmax_{q \in [0,1]}\{\bn_{\ell_1} (0\log\frac{0}{q} + 1\log\frac{1}{1-q}) \le \log(T) + 3\log\log(T)\}\\
& = \argmax_{q \in [0,1]}\{\bn_{\ell_1} \log\frac{1}{1-q} \le \log(T) + 3\log\log(T)\}.
\end{split}
\end{align}

Apparently, when $q = 1$, $\log(T) + 3\log\log(T) \le \bn_{\ell_1} \log(1/1-q) = + \infty$, hence $q$ should smaller than $1$. When $\bn_{\ell_1} \log(1/(1-q)) = \log(T) + 3\log\log(T)$, we have
\begin{align}
    \begin{split}
       \bn_{\ell_1} \log(\frac{1}{1-q}) &= \log(T\log(T)^3)\\
      \log(\frac{1}{1-q}) &= \log\bigg((T\log(T)^3)^{ 1/\bn_{\ell_1} }\bigg)\\
       \frac{1}{1-q} &= (T\log(T)^3)^{ 1/\bn_{\ell_1} }\\
       q &
       = 1 - (T\log(T)^3)^{ -1/\bn_{\ell_1} }.
    \end{split}
\end{align}
Due to $T\log(T)^3 > 1$ and $\bn_{\ell_1} > 0$, we can derive $0<(T\log(T)^3)^{ -1/\bn_{\ell_1} } < 1$ and $0 < U_{b,\ell_1}(a_k) < 1$. 

The lower confidence bound of the non-target item has
\begin{align}
    L_{b,\ell_1}(a_k) &= \argmin_{q \in [0,0]}\{\bn_{\ell_1} D_{KL}(\hat{\bc}_{b,\ell_1}(a_k) \Vert q) \le \log(T) + 3\log\log(T)\} = 0.
\end{align}
Remember the attacker returns $\tilde{\bc}_{k}^t = 1$ if $a_{k} = \tilde{a}$ and $a_k\in\bR_t$. Thus, the total collected click number of target item is $\bc_{b,\ell_1}(\tilde{a}) = 16\log(T)$ and $\hat{\bc}_{b,\ell_1}(\tilde{a}) = 1$. We can further deduce the upper confidence bound of the target item as
\begin{align}
    \begin{split}
        U_{b,\ell_1}(\tilde{a}) &= \argmax_{q \in [1,1]}\{\bn_{\ell_1} D_{KL}(\hat{\bc}_{b,\ell_1}(\tilde{a}) \Vert q) \le \log(T) + 3\log\log(T)\} = 1.
    \end{split}
\end{align}
The lower confidence bound of the target item has
\begin{align}
    \begin{split}
        L_{b,\ell_1}(\tilde{a}) &= \argmin_{q \in [0,1]}\{\bn_{\ell_1} D_{KL}(\hat{\bc}_{b,\ell_1}(\tilde{a}) \Vert q) \le \log(T) + 3\log\log(T)\} = 1\\
        & = \argmin_{q \in [0,1]}\{\bn_{\ell_1} (1\log\frac{1}{q} + (1 - 1)\log\frac{1 - 1}{1-q}) \le \log(T) + 3\log\log(T)\}\\
        & = \argmin_{q \in [0,1]}\{\bn_{\ell_1} \log\frac{1}{q} \le \log(T) + 3\log\log(T)\}\\& = 1.
    \end{split}
\end{align}
Here finish the proof of Lemma \ref{lemma1}.
\end{proof}

                    \begin{proof}[Proof of Theorem \ref{theorem1}] Consider the attacker implements attack Algorithm \ref{alg2} with $T_1 = 16L\log(T)$.
                    With the knowledge of Lemma \ref{lemma1}, we can obtain when the BatchRank begins to split the first batch $B_{1,\ell_1} = \bD$, the lower confidence bound of every non-target item satisfies $L_{b,\ell_1}(a_k) = 0$, and the lower confidence bound of the target item satisfies $L_{b,\ell_1}(\tilde{a}) = 1$. Therefore $\tilde{a}$ would be ranked at the first position because it has the highest lower confidence bound (line 6 in UpdateBatch). The BatchRank starts comparing $L_{b,\ell_1}(a_k) $ and $ \max_{a_k \in B_{k}^-}U_{b,\ell_1}(a_k)$ for $k=1$ to $K-1$ (line 10 in UpdateBatch). Owing to $L_{b,\ell_1}(\tilde{a}) = 1 > U_{b,\ell_1}(a_k)$ and $L_{b,\ell_1}(a_k) < U_{b,\ell_1}(a_k) $, the split point is $s=1$ (line 11 in Updatebatch). After the split action, the BatchRank would derive two sub-batches $B_{2,\ell_2} = \{\tilde{a}\}$ and $B_{3,\ell_3} = \bD\backslash\tilde{a}$. Sub-batch $B_{2,\ell_2}$ contains the first position of $\bR_t$ (i.e., $\ba_1^t$) and $B_{3,\ell_3}$ contains the rest of the positions of $\bR_t$ (i.e., $\ba_2^t,...,\ba_K^t$). Sub-batch $B_{2,\ell_2}$ would not split until round $T$ because it only contains a position and an item. This implies after round $16L\log(T)$, the target item would always be placed at the first position of $\bR_t$ until round $T$ is over, i.e., $\bE[\bbN_T(\tilde{a})] \ge T - 16L\log(T)$. 
Due to the click number in each round being at most $K$, the cost in one round is at most $K$. Hence, the cost of Algorithm \ref{alg2} can be bounded by $\bbC \le KT_1$.

Based on the above results, we conclude that Algorithm \ref{alg2} can efficiently attack BatchRank when $T_1 = 16L\log(T)$. Here finish the proof of Theorem \ref{theorem1}.
\end{proof}

\newpage

\section{Proof of Theorem \ref{theorem2}}

\subsection{Introduction of TopRank}
We here specifically illustrate details of
the TopRank. The pseudo-code of the TopRank is provided.
\begin{algorithm}
\renewcommand{\algorithmicrequire}{\textbf{Input:}}
\renewcommand{\algorithmicensure}{\textbf{Output:}}
	\caption{The TopRank \citep{Lattimore2018TopRankAP}}
	\label{alg4}
	\begin{algorithmic}[1]
            \STATE \textbf{Input:} Graph $G_1 = \emptyset$, round number $T$, $c = \frac{4\sqrt{2/\pi}}{\text{erf}(\sqrt{2})} \approx 3.43$
             \STATE \textbf{for} $t = 1:T$ \textbf{do}
             \STATE \quad Set $d = 0$
             \STATE \quad \textbf{while} $\bD\backslash\bigcup_{c = 1}^d\bP_{tc} \neq \emptyset$ \textbf{do}
             \STATE \quad \quad Set $d = d +1$
             \STATE \quad \quad Set $\bP_{td} = \min_{G_t}\Big( \bD \backslash \bigcup_{c = 1}^{d - 1}\bP_{tc} \Big)$
             \STATE \quad Choose $\bR_t$ uniformly at random from $\bP_{t1},...,\bP_{td}$
             \STATE \quad Observe click feedback $\bC_{t} = (\bc_{1}^t,...,\bc_{L}^t)$
             \STATE \quad \textbf{for} $(i,j) \in [L]^2$ \textbf{do}
             \STATE \quad \quad \textbf{if} $a_i,a_j\in \bP_{td}$ for some $d$ \textbf{then}
             \STATE \quad\quad\quad Set $U_{tij} = \bc_{i}^t - \bc_{j}^t$
             \STATE \quad\quad \textbf{else}
             \STATE \quad\quad\quad Set $U_{tij} = 0$
             \STATE \quad \quad Set $S_{tij} = \sum_{s = 1}^t U_{tij}$ and $N_{tij} = \sum_{s = 1}^t \vert U_{tij} \vert$
             \STATE \quad Set $G_{t+1} = G_t \bigcup \Big\{(a_j,a_i):S_{tij} \ge \sqrt{2N_{tij}\log(\frac{c}{\delta}\sqrt{N_{tij}})}$ and $N_{tij} > 0\Big\}$
	\end{algorithmic}  
\end{algorithm}

The TopRank would begin with a blank graph $G_1 \subseteq [L]^2$. A directional edge $(a_j,a_i)\in G_t$ denotes the TopRank believes item $a_i$'s attractiveness is larger than item $a_j$. Let $\min_{G_t}(\bD\backslash\bigcup_{c = 1}^{d-1}\bP_{tc}) = \{a_i \in \bD\backslash\bigcup_{c = 1}^{d-1}\bP_{tc} : (a_i,a_j)\not\in G_t\ \text{for}\ \text{all}\ a_j \in \bD\backslash\bigcup_{c = 1}^{d-1}\bP_{tc}\}$. The algorithm would begin from round $1$ to round $T$. In each round, the TopRank would establish blocks $\bP_{t1},...,\bP_{td}$ via the graph $G_t$. Items in block $\bP_{t1}$ would be placed randomly at the first $\vert\bP_{t1}\vert$ positions in $\bR_t$, and items in $\bP_{t2}$ would be placed randomly at the next $\vert\bP_{t2}\vert$ positions, and so on. In each round, after deriving click feedback $\bC_t$, the TopRank would compute $U_{tij} = \bc_{i}^t - \bc_{j}^t$ if item $a_i$ and item $a_j$ are in the same block, otherwise, $U_{tij} = 0$. Afterward, the TopRank would compute $S_{tij} = \sum_{s=1}^t U_{sij}$ and $N_{tij} = \sum_{s=1}^t \vert U_{sij} \vert$ and establish edge $(a_j,a_i)$ if $S_{tij} \ge \sqrt{2N_{tij}\log(\frac{c}{\delta}\sqrt{N_{tij}})}$ and $N_{tij} > 0$. Without the attacker interference, the graph would not contain any cycle with probability at least $1-\delta L^2$, if the graph contains at least one cycle the TopRank would behave randomly \citep{Lattimore2018TopRankAP}. Parameter $\delta$ would be set as $\delta = 1/T$.

\subsection{Missing example in section 4.3}

\paragraph{Example 3.}
The process of the attack is shown in Figure \ref{fig4}.
Consider the total item set $\bD = \{a_1,a_2,a_3\}$ with $3$ items. The length of the list $\bR_t$ is $K=2$ and the target item is $\tilde{a} = a_1$. The TopRank would start with block $\bP_{11} = \bD$ and $d = 1$ because the graph contains no edges at the beginning. In the first $T_1$ rounds, the attacker receives click feedback $\bC_{t}$ and modifies click feedback $\tilde{\bc}_{k}^t = 1$ if $a_k = \tilde{a}$ and $a_k\in\bR_t$, otherwise $\tilde{\bc}_{k}^t = 0$. After $T_1$, the edges $(a_k,\tilde{a}),\ k=2,3$ are established simultaneously. In the last $T - T_1$ rounds, the block $\bP_{t1}$ would only contain $\tilde{a}$ and $\tilde{a}$ would always be placed at the first position of $\bR_t$. Due to TopRank would only compare items' attractiveness in the same block, the edges from $\tilde{a}$ to $a_{k} \neq \tilde{a}$ would never be established and cycle would appear in $G_t$ with very low probability (will be explained in the proof of Theorem \ref{theorem2} in the appendix). 

\begin{figure}
\centering  
\includegraphics[width=0.7\linewidth]{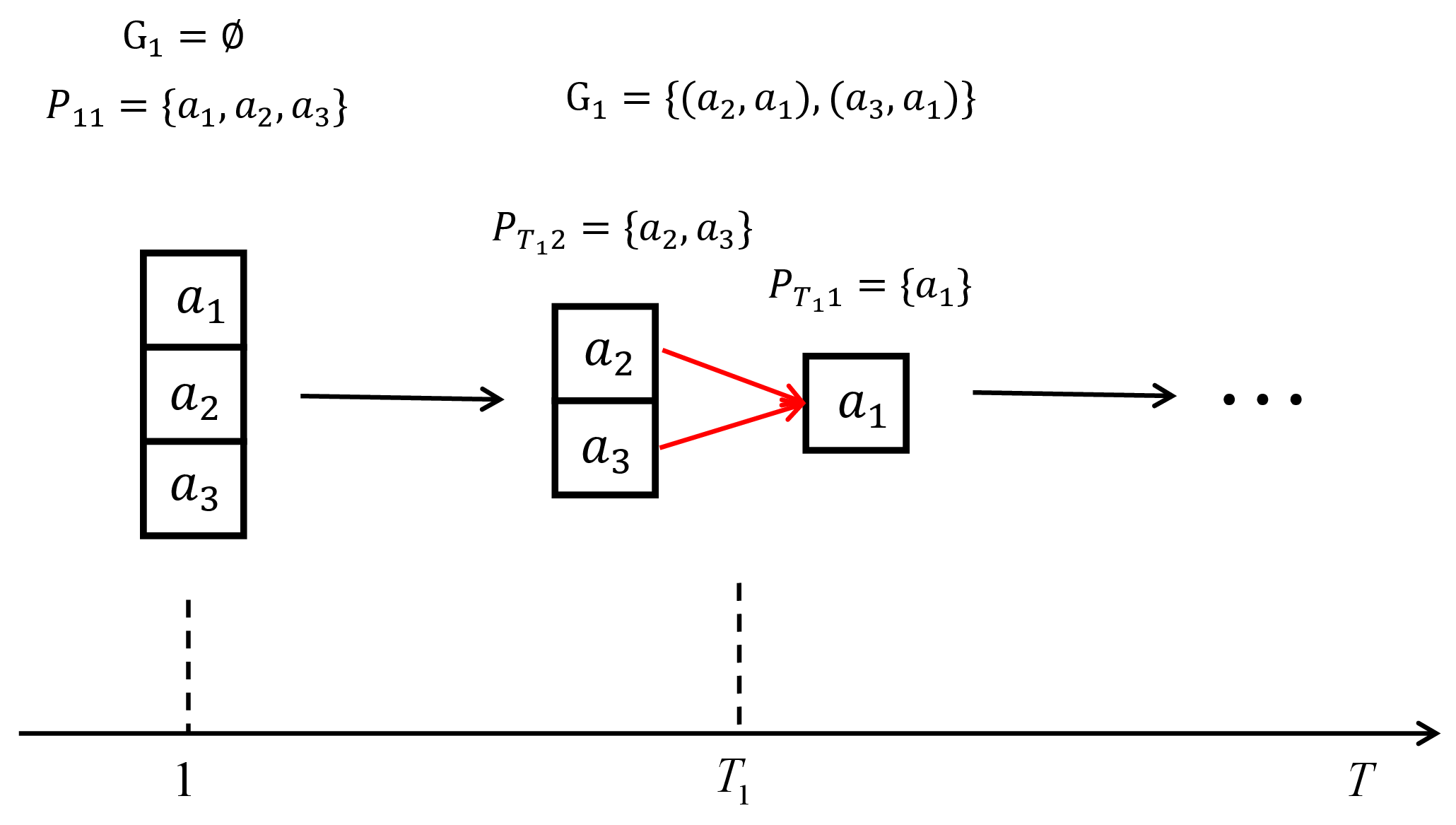}
\caption{Process of the Algorithm \ref{alg2} attacks TopRank.}
\label{fig4}
\vspace{-2mm}
\end{figure}

\subsection{Proof of Theorem \ref{theorem2}}
The proof of Theorem \ref{theorem2} relies on the following lemmas.

\begin{lemma} \label{lemma2} Consider the TopRank is under the attack of Algorithm \ref{alg2}. Denotes $a_i = \tilde{a}$ as the target item and $a_j \neq \tilde{a}$ as non-target items. When $\sum_{t=1}^{T_1} \bone\{\tilde{a} \in \bR_t\} = 4\log(c/\delta)$, then $S_{T_1ij} \ge \sqrt{2N_{T_1ij}\log(\frac{c}{\delta}\sqrt{N_{T_1ij}})}$ and $N_{T_1ij} > 0$ are satisfied and
edges from non-target items to target item (i.e., $(a_j,a_i)$, $a_j \not = a_i$) are established simultaneously.
\end{lemma}

\begin{proof}[Proof of Lemma \ref{lemma2}]
Note that the TopRank sets $U_{tij} = \bc_{i}^t - \bc_{j}^t$ if $a_i,a_j \in \bP_{td}$ for some $d$, otherwise, $U_{tij} = 0$. According to attack Algorithm \ref{alg2}, the TopRank would receive $\bc_{i}^t = 1$ ($\bc_{i}^t$ is generates by the target item) if $\tilde{a} \in \bR_t$ and $\bc_j^t = 0$ ($\bc_{j}^t$ is generated by non-target items) when $t \le T_1$. Based on this, we can derive
\begin{align}
    U_{tij} = \bc^t_i - \bc^t_j =1,\ t \le T_1,\ \tilde{a} \in \bR_t,\ a_i,a_j\in\bP_{td}.
\end{align}
Thus, when $\bP_{t1} = \{\bD\}$, we have
\begin{align} \label{21}
    S_{tij} = \sum_{s=1}^t U_{tij} = N_{tij} = \sum_{s = 1}^t \vert U_{tij}\vert = \sum_{s=1}^t \bone \{\tilde{a} \in \bR_t\},\quad t \le T_1.
\end{align}
In the light of (\ref{21}) and line 15 of TopRank, if $\sum_{s=1}^t \bone \{\tilde{a} \in \bR_t\} \ge \sqrt{2\sum_{s=1}^t \bone \{\tilde{a} \in \bR_t\}\log(\frac{c}{\delta}\sqrt{\sum_{s=1}^t \bone \{\tilde{a} \in \bR_t\}})}$, edges $(a_j,a_i)$ would establish. Utilizing the knowledge of the elementary algebra, we have
\begin{align}
    \begin{split}
        \sum_{s=1}^t \bone \{\tilde{a} \in \bR_t\} &\ge \sqrt{2\sum_{s=1}^t \bone \{\tilde{a} \in \bR_t\}\log\bigg(\frac{c}{\delta}\sqrt{\sum_{s=1}^t \bone \{\tilde{a} \in \bR_t\}}\bigg)}\\
        \bigg( \sum_{s=1}^t \bone \{\tilde{a} \in \bR_t\} \bigg)^2 &\ge 2 \sum_{s=1}^t \bone \{\tilde{a} \in \bR_t\} \Bigg(\log(\frac{c}{\delta}) 
 + \log\bigg(\sqrt{\sum_{s=1}^t \bone \{\tilde{a} \in \bR_t\}}\bigg)\Bigg)\\
       \frac{1}{2} \sum_{s=1}^t \bone \{\tilde{a} \in \bR_t\} - \log\bigg(\sqrt{\sum_{s=1}^t \bone \{\tilde{a} \in \bR_t\}}\bigg)  &\ge \log(\frac{c}{\delta})\\
      \sum_{s=1}^t \bone \{\tilde{a} \in \bR_t\} &\ge 4\log(\frac{c}{\delta}).
    \end{split}
\end{align}
The second inequality holds because of $\sum_{s=1}^t \bone \{\tilde{a} \in \bR_t\} > 0$. The fourth inequality holds because of $(1/4)x > log(\sqrt{x})$ when $x > 0$. Thus, when $\sum_{s=1}^t \bone \{\tilde{a} \in \bR_t\} \ge 4\log(c/\delta)$ and $t\le T_1$, edges $(a_j,a_i)$ would establish simultaneously. We here finish the proof of Lemma \ref{lemma2}.
\end{proof}

\begin{lemma} \label{lemma6} Suppose input $T_1 = \frac{4\log(c/\delta)}{\frac{K}{L} + (1 - \sqrt{1 + 8K/L})/4}$, then with probability at least $1-\delta/c$, the TopRank would achieve  $\sum_{t=1}^{T_1} \bone\{\tilde{a} \in \bR_t\} > 4\log(c/\delta) $. 
\end{lemma}

\begin{proof}[Proof of Lemma \ref{lemma6}] According to the previous discussion, we can separate $T_1$ into two periods $P_1$ and $P_2$ (i.e., $T_1 = P_1 + P_2$). In period one $G_t = \emptyset$ and in period two $G_t$ only contains edges from non-target items to the target item. Based on the TopRank property, in period one $ P(\tilde{a} \in \bR_t\vert t \le P_1) = K/L$ and in period two $ P(\tilde{a} \in \bR_t \vert P_1 + 1 \le t \le T_1) = 1$.
Define a Bernoulli distribution $X$ that satisfies $X = 1$ with probability $K/L$. With the help of the Hoeffding inequality, we can derive
\begin{align}
    P\bigg(\sum_{t=1}^{T_1} X_t - \frac{K}{L}T_1 \le -aT_1 \bigg) \le e^{-T_1 a^2/2}.
\end{align}
Set $T_1 = \frac{4\log(c/\delta)}{\frac{K}{L} + (1 - \sqrt{1 + 8K/L})/4}$ and $a = -(1 - \sqrt{1 + 8K/L})/4$. We can derive
\begin{align}
     P\bigg(\sum_{t=1}^{T_1} X_t \le 4\log(c/\delta)\bigg) \le \frac{\delta}{c}.
\end{align}
Further derivation shows that
\begin{align}\label{32}
    P\bigg(\sum_{t=1}^{T_1} X_t  > 4\log(c/\delta)\bigg) > 1 - \frac{\delta}{c}.
\end{align}
Follows the definition of the TopRank, one has  
\begin{align}\label{33}
\begin{split}
     P\bigg(\sum_{t=1}^{P_1} \bone \{\tilde{a} \in \bR_t\} + \sum_{t=P_1 + 1}^{T_1} \bone \{\tilde{a} \in \bR_t\} > 4\log(c/\delta)\bigg) &= P\bigg(\sum_{t=1}^{P_1} \bone \{\tilde{a} \in \bR_t\} + P_2 > 4\log(c/\delta)\bigg) 
     \\&\ge P\bigg(\sum_{t=1}^{T_1} X_t > 4\log(c/\delta)\bigg)
     \end{split}
\end{align}
where the first equation holds because $\sum_{t=P_1 + 1}^{T_1} \bone \{\tilde{a} \in \bR_t\} = P_2$. The last inequality holds because $P_2 \ge \sum_{t=P_1+1}^{T_1} X_t$. Combining (\ref{32}) and (\ref{33}), we can finally get
\begin{align}
    P\bigg(\sum_{t=1}^{T_1} \bone \{\tilde{a} \in \bR_t\}  \ge 4\log(c/\delta) \bigg) > 1 - \frac{\delta}{c}
\end{align}
when $T_1 = \frac{4\log(c/\delta)}{\frac{K}{L} + (1 - \sqrt{1 + 8K/L})/4}$. Here finish the proof of Lemma \ref{lemma6}.
\end{proof}

\begin{lemma} \label{lemma3} If the attacker implements attack Algorithm \ref{alg2} and $T_1 = \frac{4\log(c/\delta)}{\frac{K}{L} + (1 - \sqrt{1 + 8K/L})/4}$, the graph $G_t$ would not contain any cycle with probability at least $1 - (L^2+1/c)\delta$.
\end{lemma}

\begin{proof}[Proof of Lemma \ref{lemma3}]  We here analyze our attack Algorithm \ref{alg2} would not case $G_t$ contains any cycle with high probability if the input $T_1 = \frac{4\log(c/\delta)}{\frac{K}{L} + (1 - \sqrt{1 + 8K/L})/4}$. Consider the attacker implementing our attack strategy from round $1$ to round $T_1$. Define $a_i = \tilde{a}$ and $a_j \neq \tilde{a}$. The attacker frauds the TopRank to believe the target item $\tilde{a}$ is clicked $\sum_{t=1}^{T_1} \bone \{\tilde{a} \in \bR_t\}$ times and non-target items are clicked 0 time in $T_1$. After $S_{tij} \ge \sqrt{2N_{tij}\log(\frac{c}{\delta}\sqrt{N_{tij}})}$ and $N_{tij} > 0$ are satisfied, the edges would be established at the same time and $\tilde{a}$ would belong to the first block (line 6 in the TopRank and Lemma \ref{lemma2} and \ref{lemma3}). Note that during $T_1$, the attacker sets $\bc_{j}^t = 0$. Thus
\begin{align}\label{29}
\begin{split}
N_{T_1ji} &= \sum_{t=1}^{T_1}\vert U_{tji}\vert = \sum_{t=1}^{T_1}\vert \bc_{j}^t - \bc_{i}^t \vert =  \sum_{t=1}^{T_1} \bone \{\tilde{a} \in \bR_t\}\\
S_{T_1ji} &= \sum_{t=1}^{T_1} U_{tji} = \sum_{t=1}^{T_1} (\bc_{j}^t - \bc_{i}^t)  = - \sum_{t=1}^{T_1} \bone \{\tilde{a} \in \bR_t\}.
\end{split}
\end{align}
Since $U_{tij}$ and $U_{tji}$ would be $0$ after $t>T_1$ (line 9-13 in TopRank), we can obtain $S_{T_1ji} = - \sum_{t=1}^{T_1} \bone \{\tilde{a} \in \bR_t\}$ and $N_{T_1ji} = \sum_{t=1}^{T_1} \bone \{\tilde{a} \in \bR_t\}$ hold when $t>T_1$. This implies the directional edges from the target item to non-target items would never establish, i.e., $-\sum_{t=1}^{T_1} \bone \{\tilde{a} \in \bR_t\} < \sqrt{2\sum_{t=1}^{T_1} \bone \{\tilde{a} \in \bR_t\}\log(\frac{c}{\delta}\sqrt{\sum_{t=1}^{T_1} \bone \{\tilde{a} \in \bR_t\}})}$). Besides, due to the received click number from non-target items being $0$ in $T_1$, the $S_{T_1}$ and $N_{T_1}$ between non-target items would be $0$. This implies the manipulation of the attacker would not influence the TopRank judgment of the attractiveness between non-target items. In other words, the TopRank under Algorithm \ref{alg2} attack can be considered as the TopRank interacts with item set $\bD\backslash \tilde{a}$ in $T-T_1$ rounds.

According to the above discussion and Lemma \ref{lemma6}, if $T_1 = \frac{4\log(c/\delta)}{\frac{K}{L} + (1 - \sqrt{1 + 8K/L})/4}$, then $S_{T_1ij} \ge \sqrt{2N_{T_1ij}\log(\frac{c}{\delta}\sqrt{N_{T_1ij}})}$ would satisfy with probability at least $1 - \delta/c$. Besides, from round $T_1+1$ to $T$, cycles would occur with probability at most $\delta L^2$. Thus graph $G_t$ would not contain cycles with probability at least $1-(L^2+1/c)\delta$ until $T$.
\end{proof}

\begin{proof}[Proof of Theorem \ref{theorem2}] 
Suppose the attacker implements attack Algorithm \ref{alg2} with input value $T_1 = \frac{4\log(c/\delta)}{\frac{K}{L} + (1 - \sqrt{1 + 8K/L})/4}$. Then, the TopRank would establish edges from non-target items to $\tilde{a}$ with probability at least $1-\delta/c$ (According to Lemma \ref{lemma2} and Lemma \ref{lemma6}). Based on the analysis in Lemma \ref{lemma3}, the cycle would appear with probability at most $(L^2+1/c)\delta$ and the first block would only contain $\tilde{a}$ till $T$. That is to say, the target item in block $\bP_{t1}$ would always be placed at the first positions after $T_1$ with probability at least $1-(L^2+1/c)\delta$. Following Algorithm \ref{alg2}, the attacker would only manipulate the returned click feedback for $T_1$ times. Thus the attack cost can be bounded by $\bbC(T) \le KT_1$.

According to the above observation, we summarize that Algorithm \ref{alg2} can efficiently attack TopRank when $T_1 = \frac{4\log(c/\delta)}{\frac{K}{L} + (1 - \sqrt{1 + 8K/L})/4}$. Here finish the proof of Theorem \ref{theorem2}.
\end{proof}

\end{document}